\documentclass[english]{article}

\usepackage[T1]{fontenc}
\usepackage[utf8]{inputenc}
\usepackage{geometry}
\usepackage{microtype}
\usepackage{graphicx}
\usepackage{subfigure}
\usepackage{algorithm,algorithmic}
\usepackage{booktabs} 
 
\usepackage{hyperref}

\usepackage{natbib}

\usepackage{amsmath}
\usepackage{amssymb}
\usepackage{mathtools}
\usepackage{amsthm}
\allowdisplaybreaks

\usepackage{mathrsfs}

\usepackage[capitalize,noabbrev]{cleveref}

\theoremstyle{plain}
\newtheorem{theorem}{Theorem}[section]
\newtheorem{proposition}[theorem]{Proposition}
\newtheorem{lemma}[theorem]{Lemma}

\theoremstyle{definition}
\newtheorem{definition}[theorem]{Definition}
\newtheorem{assumption}[theorem]{Assumption}
\theoremstyle{remark}

\usepackage[textsize=tiny]{todonotes}

\usepackage{hyperref}       
\usepackage{cleveref}       
\usepackage{url}            
\usepackage{booktabs}       
\usepackage{nicefrac}       
\usepackage{microtype}      
\usepackage{xcolor}         

\usepackage{microtype}
\usepackage{graphicx}
\usepackage{makecell}
\usepackage{caption}
\usepackage{multirow}

\usepackage{colortbl}
\definecolor{bgcolor}{rgb}{0.66,0.88,1.00}

\usepackage{tablefootnote}
\usepackage{threeparttable}

\usepackage{xspace}

\xspaceaddexceptions{[]\{\}}

\newcommand{\dataset}[1]{{\tt #1}\xspace}

\usepackage{thm-restate}

\usepackage{eqparbox}

\usepackage{amsmath,amsfonts,bm}

\def\1{\bm{1}}

\def\vone{{\bm{1}}}

\def\va{{\bm{a}}}
\def\vb{{\bm{b}}}

\makeatletter
\newcommand{\ve}{\@ifnextchar\bgroup{\velong}{{\bm{e}}}}
\newcommand{\velong}[1]{{\bm{#1}}}
\makeatother

\def\vg{{\bm{g}}}
\def\vh{{\bm{h}}}

\def\vm{{\bm{m}}}

\def\vq{{\bm{q}}}

\def\vs{{\bm{s}}}

\def\vu{{\bm{u}}}
\def\vv{{\bm{v}}}

\def\vx{{\bm{x}}}

\def\mA{{\bm{A}}}

\def\mG{{\bm{G}}}
\def\mH{{\bm{H}}}
\def\mI{{\bm{I}}}

\def\mM{{\bm{M}}}

\def\mQ{{\bm{Q}}}

\def\mV{{\bm{V}}}
\def\mW{{\bm{W}}}
\def\mX{{\bm{X}}}

\DeclareMathAlphabet{\mathsfit}{\encodingdefault}{\sfdefault}{m}{sl}
\SetMathAlphabet{\mathsfit}{bold}{\encodingdefault}{\sfdefault}{bx}{n}

\def\gC{{\mathcal{C}}}
\def\gD{{\mathcal{D}}}

\def\gI{{\mathcal{I}}}

\def\gN{{\mathcal{N}}}

\newcommand{\E}{\mathbb{E}}
\newcommand{\R}{\mathbb{R}}

\DeclareMathOperator{\sign}{sign}

\newcommand{\dotp}[2]{\left<#1, #2\right>}

\newcommand{\norm}[1]{\left\| #1 \right\|}

\newcommand{\normF}[1]{\left\| #1 \right\|_{\mathrm{F}}}

\def\eqref#1{(\ref{#1})}

\newcommand{\algname}[1]{{\sf\small#1}\xspace}
\newcommand{\beer}{{\sf\small BEER}\xspace}

\usepackage{todonotes}

\makeatletter
\allowdisplaybreaks

\title{BEER: Fast $O(1/T)$ Rate for Decentralized Nonconvex Optimization with Communication Compression}

\author{
	Haoyu Zhao\thanks{Department of Computer Science, Princeton University, Princeton, NJ 08540, USA; Email: \texttt{haoyu@princeton.edu}.} \\
	Princeton \\
	\and
	Boyue Li\thanks{Department of Electrical and Computer Engineering, Carnegie Mellon University, Pittsburgh, PA 15213, USA; Emails:
		\texttt{\{boyuel,zhizel,yuejiec\}@andrew.cmu.edu}.} \\
		CMU \\ 
		\and
    Zhize Li\footnotemark[2] \textsuperscript{,}\footnote{Corresponding author.}  \\
		CMU \\
		\and
		Peter Richt\'arik\thanks{Computer, Electrical and Mathematical Sciences and Engineering Division, King Abdullah University of Science and Technology, Thuwal 23955-6900, Kingdom of Saudi Arabia; Email: \texttt{peter.richtarik@kaust.edu.sa}.}  \\
		KAUST \\
		\and
	Yuejie Chi\footnotemark[2] \\
	CMU
}

\date{January 2022; October 2022 Revised}

\begin{document}

\maketitle

\begin{abstract}
 
Communication efficiency has been widely recognized as the bottleneck for large-scale decentralized machine learning applications in multi-agent or federated environments. To tackle the communication bottleneck, there have been many efforts to design communication-compressed algorithms for decentralized nonconvex optimization, where the clients are only allowed to communicate a small amount of quantized information (aka bits) with their neighbors over a predefined graph topology. Despite significant efforts, the state-of-the-art algorithm in the nonconvex setting still suffers from a slower rate of convergence $O((G/T)^{2/3})$ compared with their uncompressed counterpart, where $G$ measures the data heterogeneity across different clients, and $T$ is the number of communication rounds. This paper proposes \beer, which adopts communication compression with gradient tracking, and shows it converges at a \emph{faster rate} of $O(1/T)$. This significantly improves over the state-of-the-art rate, by matching the rate without compression even under arbitrary data heterogeneity. Numerical experiments are also provided to corroborate our theory and confirm the practical superiority of \beer in the data heterogeneous regime.

\end{abstract}

\medskip
\noindent\textbf{Keywords:} decentralized nonconvex optimization, communication compression, fast rate.

\section{Introduction}

Decentralized machine learning is gaining attention in both academia and industry because of its emerging applications in multi-agent systems such as the internet-of-things (IoT) and networked autonomous systems~\citep{marvasti2014optimal,savazzi2020federated}. One of the key problems in decentralized machine learning is on-device training, which aims to optimize a machine learning model using the datasets stored on (geographically) different clients, and can be formulated as a \emph{decentralized optimization} problem.

Decentralized optimization aims to solve the following optimization problem without sharing the local datasets with other clients:
\begin{align} \label{eq:prob}
    \min_{\vx\in\R^d}\left\{ f(\vx; \gD) := \frac{1}{n}\sum_{i=1}^n f(\vx; \gD_i)\right\},
\end{align}
where $f(\vx; \gD_i) := \E_{\xi_i\sim\gD_i} f(\vx; \xi_i)$ for $i\in [n]$, and $n$ is the total number of clients. Here, $\vx \in \mathbb{R}^d$ is the machine learning model, $f(\vx; \gD)$, $f(\vx; \gD_i)$, and $f(\vx; \xi_{i})$ denote the loss functions of the model $\vx$ on the entire dataset $\gD$, the local dataset $\gD_i$, and a random data sample $\xi_{i}$, respectively. Different from the widely studied distributed or federated learning setting where there is a central server to coordinate the parameter sharing across all clients, in the decentralized setting, each client can only communication with its neighbors over a communication network determined by a predefined network topology.

The main bottleneck of decentralized optimization---when it comes to large-scale machine learning applications---is communication efficiency, due to the large number of clients involved in the network~\citep{savazzi2020federated} and the enormous size of machine learning models~\citep{brown2020language}, exacerbated by resource constraints such as limited bandwidth availability and stringent delay requirements. One way to reduce the communication cost is communication compression, which only transmits compressed messages (with fewer bits) between the clients using \emph{compression operators}. The compression operators come with many design choices and offer great flexibility in different trade-offs of communication and computation in practice. Even though communication compression has been extensively applied to distributed or federated optimization with a central server \citep{stich2018sparsified,karimireddy2019error,das2020improved,li2020acceleration,gorbunov2021marina,richtarik2021ef21,richtarik20223pc,li2022soteriafl}, its use in the decentralized setting has been relatively sparse. Most of the existing approaches only apply to the strongly convex setting \citep{reisizadeh2019exact,koloskova2019decentralized,liu2020linear,kovalev2021linearly,liao2021compressed,li2021decentralized}, and only a few consider the general nonconvex setting~\citep{koloskova2019decentralizeddeep,tang2019deepsqueeze,singh2021squarm}.

\subsection{Our contributions}

This paper considers decentralized optimization with communication compression, focusing on the {\em nonconvex} setting due to its critical importance in modern machine learning, such as training deep neural networks~\citep{lecun2015deep}, word embeddings, and other unsupervised learning models~\citep{saunshi2019theoretical}. Unfortunately, existing algorithms~\citep{koloskova2019decentralizeddeep,tang2019deepsqueeze,singh2021squarm} suffer from several important drawbacks in the nonconvex setting: they need strong bounded gradient or bounded dissimilarity assumptions to guarantee convergence, and the convergence rate is order-wise slower than their uncompressed counterpart in terms of the communication rounds (see Table \ref{tab:existed-results}).

In this paper, we introduce \beer, which is a decentralized optimization algorithm with communication compression using gradient tracking. \beer not only removes the strong assumptions required in all prior works, but enjoys a faster convergence rate in the nonconvex setting. Concretely, we have the following main contributions (see Tables~\ref{tab:existed-results} and \ref{tab:our-results}).
\begin{enumerate}
    \item We show that \beer converges at a fast rate of $O(1/T)$ in the nonconvex setting, which improves over the state-of-the-art rate $O(1/T^{2/3})$ of \algname{CHOCO-SGD}~\citep{koloskova2019decentralizeddeep} and \algname{Deepsqueeze}~\citep{tang2019deepsqueeze}, where $T$ is the number of communication rounds. This matches the rate without compression even under arbitrary data heterogeneity across the clients.
    \item We also provide the analysis of \beer under the Polyak- \L ojasiewicz (PL) condition (Assumption \ref{ass:pl}), and show that \beer converges at a linear rate (see Table~\ref{tab:our-results}). Note that strong convexity implies the PL condition, and thus \beer also achieves linear convergence in the strongly convex setting.
    \item We run numerical experiments on real-world datasets and show \beer achieves superior or competitive performance when the data are heterogeneous compared with state-of-the-art baselines with and without communication compression.
\end{enumerate}

To the best of our knowledge, \beer is the {\em first} algorithm that achieves $O(1/T)$ rate without the bounded gradient or bounded dissimilarity assumptions, supported by a strong empirical performance in the data heterogeneous setting.

\begin{table*}[!t]
\renewcommand{\arraystretch}{1.5}
    \centering
 
    \begin{tabular}{|c|c|c| }
    \hline
    \bf Algorithm & \bf Convergence rate & \bf Strong assumption   \\
    \hline
        \makecell{\algname{SQuARM-SGD}  \\ \citep{singh2021squarm}} & $O\left(\frac{1}{\sqrt{nT}} + \frac{nG^2}{T}\right)$ & Bounded Gradient    \ \ \ \ \\ \hline
        \makecell{\algname{DeepSqueeze}   \\ \citep{tang2019deepsqueeze}} & $O\left(\left(\frac{G}{T}\right)^{2/3}\right)$ & Bounded Dissimilarity   \ \ \ \  \\ \hline
        \makecell{\algname{CHOCO-SGD}    \\  \citep{koloskova2019decentralizeddeep}} & $O\left(\left(\frac{G}{T}\right)^{2/3}\right)$  & Bounded Gradient \\ \hline
        \cellcolor{bgcolor} \beer (Algorithm \ref{alg:becdo})   \ \ \ \  & \cellcolor{bgcolor} $O\left(\frac{1}{T}\right)$ & \cellcolor{bgcolor} --- \\ \hline
    \end{tabular}
     \caption{Comparison of convergence rates for existing decentralized methods with communication compression in the nonconvex setting. Here, the parameter $G$ refers the quantity either in the bounded gradient assumption $\E_{\xi_i\sim \gD_i}\norm{\nabla f(\vx, \xi_i)}^2 \le G^2$ or the bounded dissimilarity assumption $\E_{i}\norm{\nabla f(\vx, \gD_i) - \nabla f(\vx, \gD)}^2 \le G^2$, both of which are very strong assumptions  (the bounded dissimilarity assumption is slightly weaker) that \beer does {\em not} require. All algorithms support the use of stochastic gradients with bounded local variance at local clients.}
    \label{tab:existed-results}

\end{table*}
    
   \begin{table*}[!h]
   	\renewcommand{\arraystretch}{1.5}
   	\centering
    \begin{tabular}{|c|c|c|}
    \hline
    \bf Assumptions & \bf Convergence rate & \bf Theorem \\
    \hline
        \makecell{$f_i$ is $L$-smooth} & $\frac{1}{T}\sum_{t=0}^{T-1} \E \norm{\nabla f(\bar{\vx}^t)}^2 \le \frac{2(\Phi_0 - \Phi_T)}{\eta T}$ ~\tnote{(5)}\ \ \ \  & Theorem~\ref{thm:nonconvex-full} \\
    \hline
        \makecell{$f_i$ is $L$-smooth \\ $f$ satisfies PL condition} & $\Phi_T \le (1-\mu \eta)^T \Phi_0$ ~\tnote{(6)}\ \ \ \  & Theorem~\ref{thm:pl-full} \\
    \hline
    \end{tabular}
    \caption{Summary of the established convergence rates for the proposed \beer algorithm in the nonconvex setting. Here, $\bar{\vx}^t$ is the average model of all clients, $\eta$ is the step size, $\Phi_t$ is the Lyapunov function (cf.~\eqref{eq:lf}), and $\mu$ is the PL-condition parameter (cf.~Assumption \ref{ass:pl}). We do not assume the bounded gradient or bounded dissimilarity assumption.  }   	\label{tab:our-results}
 
\end{table*}

\subsection{Related works}
In this section, we review closely related literature on decentralized optimization, communication-efficient algorithms, and communication compression.

\paragraph{Decentralized optimization}

Decentralized optimization, which is a special class of linearly constrained (consensus constraint) optimization problems, has been studied for a long time~\citep{bertsekas2015parallel,glowinski1975solution}. Many centralized algorithms can be intuitively converted into decentralized counterparts by using 
gossip averaging~\citep{kempe2003gossip,xiao2004fast}, which mixes parameters from neighboring clients to enforce consensus. 

However, direct applications of gossip averaging often lead to either slow convergence or high error floors  \citep{nedic2009distributed}, and many fixes have been proposed in response \citep{shi2015extra,yuan2018exact,qu2017harnessing,di2016next,nedic2017achieving}. Among them, gradient tracking \citep{qu2017harnessing,di2016next,nedic2017achieving}, which applies the idea of dynamic average consensus \citep{zhu2010discrete} to global gradient estimation, provides a systematic approach to reduce the variance and has been successfully applied to decentralize many algorithms with faster rates of convergence \citep{li2020communication,sun2019distributed}. For nonconvex problems, a small sample of gradient tracking aided algorithms include \algname{GT-SAGA}~\citep{xin2021fast}, \algname{D-GET}~\citep{sun2020improving},  \algname{GT-SARAH}~\citep{xin2020fast}, and
\algname{DESTRESS}~\citep{li2021destress}. Our \beer algorithm also leverages gradient tracking to eliminate the strong bounded gradient and bounded dissimilarity assumptions.

\paragraph{Communication-efficient algorithms}
While decentralized optimization is a classical topic, the focus on communication efficiency is relatively new due to the advances in large-scale machine learning. Roughly speaking, there are primarily two kinds of approaches to reduce communication cost: 1) \emph{local methods}: in each communication round, clients run multiple local update steps before communicating, in the hope of reducing the number of communication rounds;
2) \emph{compressed methods}: clients send compressed communication messages, in the hope of reducing the communication cost per communication round.

Both categories have received significant attention in recent years. For local methods, a small sample of examples include \algname{FedAvg}~\citep{mcmahan2017communication}, \algname{Local-SVRG}~\citep{gorbunov2020local}, \algname{SCAFFOLD}~\citep{karimireddy2020scaffold} and \algname{FedPAGE}~\citep{zhao2021fedpage}.
On the other hand, many compressed methods are proposed recently such as \citep{alistarh2017qsgd, khirirat2018distributed,tang2018communication,stich2018sparsified,koloskova2019decentralizeddeep, li2020acceleration, gorbunov2021marina, li2021canita, richtarik2021ef21,fatkhullin2021ef21,zhao2021faster}.
In this paper, we will adopt the second approach based on communication compression to enhance communication efficiency.  

\paragraph{Decentralized nonconvex optimization with compression}

As discussed earlier and summarized in Table~\ref{tab:existed-results}, there have been limited existing works on decentralized nonconvex optimization with communication compression. In particular, \algname{SQuARM-SGD}~\citep{singh2021squarm} can be viewed as \algname{CHOCO-SGD} with momentum, but its theoretical convergence rate is slower than the original \algname{CHOCO-SGD}. \algname{Deepsqueeze}~\citep{tang2019deepsqueeze} and \algname{CHOCO-SGD} have a close relationship, where \algname{Deepsqueeze} can be viewed as decentralized SGD (\algname{DSGD} or \algname{D-PSGD}) with the explicit error feedback framework, and \algname{CHOCO-SGD} uses control variables to implicitly handle the compression error.

\paragraph{Notation} Throughout this paper, we use boldface letters to denote vectors, e.g., $\vx \in \R^d$. Let $[n]$ denote the set $\{1,2,\cdots,n\}$, $\vone$ be the all-one vector, $\mI$ be the identity matrix, $\norm{\cdot}$ denote the Euclidean norm of a vector, and $\norm{\cdot}_{\mathrm{F}}$ denote the Frobenius norm of a matrix.
Let $\dotp{\vu}{\vv}$ denote the standard Euclidean inner product of two vectors $\vu$ and $\vv$. In addition, we use the standard order notation $O(\cdot)$ to hide absolute constants.

\section{Problem Setup}\label{sec:setup}

In this section, we formally define the decentralized optimization problem with communication compression, and introduce a few important quantities and assumptions that will be used in developing our algorithm and theory. 

\subsection{Decentralized optimization} 
The goal of decentralized optimization is to solve
\begin{equation*}
\min_{\vx\in\R^d}\left\{ f(\vx) := \frac{1}{n}\sum_{i=1}^n f_i(\vx)\right\},
\end{equation*}
where $n$ is the number of clients, $f(\vx)$ is the global objective function, and $f_i(\vx):=f(\vx; \gD_i) := \E_{\xi_i\sim\gD_i} f(\vx; \xi_i)$ is the local objective function, with $\vx$ the parameter of interest, $\xi_i$ a random data sample drawn from the local dataset $\gD_i$. 

In the decentralized setting, the clients can only communicate with their local neighbors over a prescribed network topology, which is specified by an undirected weighted graph $\mathcal{G}([n], E)$. Here, each node in $[n]$ represents a client, and $E$ is the set of possible communication links between different clients. Information sharing across the clients is implemented mathematically by the use of a mixing matrix $\mW =[w_{ij}] \in  [0,1]^{n\times n} $, which is defined in accordance with the network topology: we assign a positive weight $w_{ij}$ for any $(i,j)\in E$ and $w_{ij}=0$ for all $(i,j)\notin E$. We make the following standard assumption on the mixing matrix~\citep{nedic2018network}.
\begin{assumption}[Mixing matrix]\label{ass:mixing}
The mixing matrix $\mW=[ w_{ij} ]\in [0,1]^{n\times n}$ is symmetric ($\mW^\top = \mW$) and doubly stochastic ($\mW \vone = \vone, \vone^{\top}\mW = \vone^{\top}$). Let its eigenvalues be
    $1 = |\lambda_1(\mW)| > |\lambda_2(\mW)| \ge \cdots \ge |\lambda_n(\mW)|$. The spectral gap is denoted by 
    \begin{equation}\label{eq:rho}
    \rho := 1- |\lambda_2(\mW)| \in (0,1]. 
    \end{equation}
\end{assumption}
The spectral gap of a mixing matrix is closely related to the network topology, see~\citet{nedic2018network} for its scaling with respect to the network size (i.e. the number of clients $n$) for representative network topologies.

\subsection{Compression operators} 

Compression, in the forms of quantization or sparsification, can be used to reduce the total communication cost. We now introduce the notion of a randomized {\em compression operator}, which is widely used in the decentralized/federated optimization literature, e.g. \citet{tang2018communication,stich2018sparsified,koloskova2019decentralizeddeep,richtarik2021ef21,fatkhullin2021ef21}.
\begin{definition}[Compression operator]\label{def:comp}
	A randomized map $\gC: \R^d\mapsto \R^d$ is an $\alpha$-compression operator  if for all $\vx\in \R^d$, it satisfies
	\begin{equation}\label{eq:comp}
	\E\left[\norm{\gC(\vx)-\vx}^2\right]\leq (1-\alpha)\norm{\vx}^2.
	\end{equation}
	In particular, no compression ($\gC(\vx)\equiv \vx$) implies $\alpha=1$.
\end{definition}
Compared with the \emph{unbiased compression operator} used in, e.g.,~\citet{alistarh2017qsgd, khirirat2018distributed, mishchenko2019distributed,li2020unified}, the compression operator in Definition \ref{def:comp} does not impose the additional constraint on the expectation such that $\E [\gC(\vx)]=\vx$. Besides, it is always possible to convert an \emph{unbiased compression operator} into a biased one satisfying Definition~\ref{def:comp}. In particular, for an unbiased compression operator $\gC': \R^d\mapsto \R^d$ that satisfies
	\begin{equation*}
	\E[\gC'(\vx)]=\vx, \qquad \E\left[\norm{\gC'(\vx)-\vx}^2\right]\leq \omega\norm{\vx}^2,
	\end{equation*}
we can construct a \emph{biased} compression operator $\gC: \gC(\vx) = \frac{\gC'(\vx)}{1+\omega}$ and the new compression operator satisfies Definition \ref{def:comp} with $\alpha = \frac{1}{1+\omega}$. Thus, Definition \ref{def:comp} is a generalization of the {unbiased compression operator} that allows \emph{biased} compression.

\subsection{Assumptions on functions} 

We now state the assumptions on the functions $\{f_i\}$ and $f$. Throughout this paper, we assume that $f^* = \min_{\vx} f(\vx)$ exists and $f^* > -\infty$.

In the nonconvex setting, we assume that the functions $\{f_{i}\}_{i\in [n]}$ are arbitrary functions that satisfy the following standard smoothness assumption. 
\begin{assumption}[Smoothness]\label{ass:smooth}
	The function $f$ is $L$-smooth if there exists $L \ge 0$ such that
	\[\norm{ \nabla f(\vx_1) - \nabla f(\vx_2)} \le L\norm{\vx_1 - \vx_2}, \forall \vx_1, \vx_2\in \R^d.\]
\end{assumption}

In addition, we allow local computation to be performed via stochastic gradient updates, where
$\tilde\nabla f_i(\vx) := \nabla f_i(\vx;\xi_i)$
denotes a local stochastic gradient computed via a sample $\xi_i$ drawn i.i.d. from $\gD_i$, and $\tilde\nabla_b f_i(\vx) := \frac{1}{b}\sum_{j=1}^b \nabla f_i(\vx;\xi_{i,j})$ denotes the stochastic gradient computed by a minibatch with size $b$ drawn i.i.d. from $\gD_i$. We assume $\tilde\nabla f_i(\vx)$ and $\tilde\nabla_b f_i(\vx) $ have bounded variance, which is again standard in the decentralized/federated optimization literature~\citep{mcmahan2017communication,karimireddy2020scaffold,koloskova2019decentralizeddeep}.

\begin{assumption}[Bounded variance]\label{ass:bounded-variance}
	There exists a constant $\sigma \ge 0$ such that for all $i\in [n]$ and $\vx\in\R^d$,
	\[\E\norm{\tilde\nabla f_i(\vx) - \nabla f_i(\vx)}^2 \le \sigma^2.\]
 For a stochastic gradient with minibatch size $b$, we have
	\[\E\norm{\tilde\nabla_b f_i(\vx) - \nabla f_i(\vx)}^2 \le \frac{\sigma^2}{b}.\]
\end{assumption}

In addition, we consider the setting when the function $f$ additionally satisfies the following Polyak- \L ojasiewicz (PL) condition~\citep{polyak1963gradient}, which can lead to fast linear convergence even when the function is nonconvex.
\begin{assumption}[PL condition]\label{ass:pl}
    There exists some constant $\mu > 0$ such that for any $\vx\in\R^d$,
    \[\norm{\nabla f(\vx)}^2 \ge 2\mu(f(\vx) - f^*).\]
\end{assumption}
Note that the PL condition is a weaker assumption than strong convexity, which means that if the objective function $f$ is $\mu$-strongly convex, then the PL condition also holds with the parameter $\mu$.

\section{Proposed Algorithm}

In this section, we introduce our proposed algorithm \beer for decentralized nonconvex optimization with compressed communication. Before embarking on the description of \beer, we introduce some convenient matrix notation. Since in a decentralized setting, the parameter estimates at different clients are typically different, we use 
$\mX = [\vx_1,\vx_2,\dots,\vx_n]$ 
to denote the collection of parameter estimates from all clients, where $\vx_i$ is from client $i$. The average of $\{\vx_i\}_{i\in [n]}$ is denoted by $\bar\vx := \frac{1}{n}\mX\vone$. Other quantities are defined similarly. With slight abuse of notation, we define  
\begin{align*}
\nabla F(\mX) & := [\nabla f_1(\vx_1), \nabla f_2(\vx_2),\ldots,\nabla f_n(\vx_n)] \in \R^{d\times n},
\end{align*} 
which collects the local gradients computed at the local parameters. Similarly, the stochastic variant is defined as $\tilde\nabla_b F(\mX) := [\tilde\nabla_b f_1(\vx_1), \tilde\nabla_b f_2(\vx_2),\dots,\tilde\nabla_b f_n(\vx_n)]$. We also allow the compression operator to take vector values, which are applied in a column-wise fashion, i.e., $\gC(\mX) := [\gC(\vx_1), \ldots, \gC(\vx_n)]\in\R^{d\times n}$.

We now proceed to describe \beer, which is detailed in Algorithm~\ref{alg:becdo} using the matrix notations introduced above. At the $t$-th iteration, \beer maintains the current model estimates $\mX^t$ and the global gradient estimates $\mV^t$ across the clients. At the crux of its design, \beer also tracks and maintains two control sequences $\mH^t$ and $\mG^t$ that serve as compressed surrogates of $\mX^t$ and $\mV^t$, respectively. In particular, these two control sequences are updated by aggregating the received compressed messages alone (cf.~Line~\ref{line:update-h} and Line~\ref{line:update-g}).

It then boils down to how to carefully update these quantities in each iteration with communication compression. To begin, note that for each client $i$, \beer not only maintains its own parameters $\{\vx
_i^t, \vv_i^t, \vh_i^t, \vg_i^t\}$, but also the control variables from its neighbors, namely, $\{\vh_j^t\}_{j\in\gN(i)}$ and $\{\vg_j^t\}_{j\in \gN(i)}$.

\paragraph{Update the model estimate} Each client $i$ first updates its model $\vx^{t+1}_i$ according to Line~\ref{line:update-x}. By thinking of $\{\vh_j^t\}_{j\in\gN(i)}$ as a surrogate of $\{\vx_j^t\}_{j\in\gN(i)}$, the second term aims to achieve better consensus among the clients through mixing, while the last term performs a gradient descent update. 

\paragraph{Update the global gradient estimate} Each client $i$ updates the global gradient estimate $\vv_i^{t+1}$ according to Line~\ref{line:update-v}, where the last correction term---based on the difference of the gradients at consecutive models---is known as a trick called \emph{gradient tracking}~\citep{qu2017harnessing,di2016next,nedic2017achieving}. The use of gradient tracking is critical: as shall be seen momentarily, it contributes to the key difference from \algname{CHOCO-SGD} that enables the fast rate of $O(1/T)$ without any bounded dissimilarity or bounded gradient assumptions. Indeed, if we remove the control sequence $\mG^t$ and substitute Lines \ref{line:update-v}-\ref{line:update-g} by $\mV^{t+1} = \tilde \nabla_b F(\mX^{t+1})$, we recover \algname{CHOCO-SGD} from \beer.

\paragraph{Update the compressed surrogates with communication} To update $\{\vh_j^t\}_{j\in\gN(i)}$, each client $i$ first computes a compressed message $\vq_{h,i}^{t+1}$ that encodes the difference $\vx^{t+1}_i-\vh_i^{t}$, and broadcasts to its neighbors (cf.~Line~\ref{line:compress-x}). Then, each client $i$ updates $\{\vh_j^t\}_{j\in\gN(i)}$ by aggregating the received compressed messages $\{\vq_{h,j}^{t+1}\}_{j\in\gN(i)}$ following Line~\ref{line:update-h}. The updates of $\{\vg_j^t\}_{j\in\gN(i)}$ can be performed similarly. Moreover, all the compressed messages can be sent in a single communication round at one iteration, i.e., the communications in Lines~\ref{line:compress-x} and ~\ref{line:compress-v} can be performed at once. This leverages \algname{EF21}~\citep{richtarik2021ef21} for communication compression, which is a \emph{better and simpler} algorithm that deals with {biased} compression operators compared with the error feedback (or error compensation, \algname{EF/EC}) framework~\citep{karimireddy2019error,stich2020error}. Using the control sequence $\mG^t$, \beer does not need to apply \algname{EF/EC} explicitly and can deal with the error implicitly.

\begin{algorithm*}[!t]
\caption{\beer: Better comprEssion for dEcentRalized optimization}
\label{alg:becdo}
\begin{algorithmic}[1]
    \STATE {\bfseries Input:} Initial point $ \mX^0= \vx_0\vone^{\top}$, $\mG^0 = \mathbf{0}$, $\mH^0 = \mathbf{0}$, $\mV^0 =\nabla F(\mX_0)$, step size $\eta$, mixing step size $\gamma$, minibatch size $b$. \\
    \FOR{$t=0,1,\dots$}
        \STATE $\mX^{t+1} = \mX^t + \gamma \mH^t (\mW - \mI)- \eta \mV^t$ \label{line:update-x}
        \STATE $\mQ_h^{t+1} = \gC(\mX^{t+1} - \mH^{t})$ \COMMENT{client $i$ sends $\vq_{h,i}^{t+1}$ to all its neighbors} \label{line:compress-x}
        \STATE $\mH^{t+1} = \mH^t + \mQ_h^{t+1}$ \label{line:update-h}
        \STATE $\mV^{t+1} = \mV^t + \gamma \mG^{t} (\mW - \mI) + \tilde\nabla_b F(\mX^{t+1}) - \tilde\nabla_b F(\mX^{t})$ \label{line:update-v}
        \STATE $\mQ_g^{t+1} = \gC(\mV^{t+1} - \mG^{t})$ \COMMENT{client $i$ sends $\vq_{g,i}^{t+1}$ to all its neighbors} \label{line:compress-v}
        \STATE $\mG^{t+1} = \mG^t + \mQ_g^{t+1}$ \label{line:update-g}
    \ENDFOR
\end{algorithmic}
\end{algorithm*}

\section{Convergence Guarantees}

In this section, we show the convergence guarantees of \beer under different settings: the $O(1/T)$ rate in the nonconvex setting in Section \ref{sec:nonconvex}, and the improved linear rate under the PL condition (Assumption \ref{ass:pl}) in Section \ref{sec:pl}. In Section \ref{sec:proof-sketch}, we briefly sketch the proof.

Our convergence guarantees are based on an appropriately designed Lyapunov function, given by
      \begin{equation}\label{eq:lf}
        \Phi_t = \E f(\bar\vx^t) - f^* + \frac{c_1 L}{n}\Omega_1^t + \frac{c_2 \rho^2}{n L}\Omega_2^t + \frac{c_3 L}{n}\Omega_3^t + \frac{c_4 \rho^4}{nL}\Omega_4^t,
    \end{equation}
where the choice of constants $\{c_i\}_{i=1}^4$ might be different from theorem to theorem, $\E f(\bar\vx^t) - f^* $ represents the sub-optimality gap, and the errors $\{\Omega_i^t\}_{i=1}^4$ are defined by
\begin{align}
\Omega_1^t & := \E\normF{\mH^t - \mX^t}^2,  \quad ~~
\Omega_2^t  := \E\normF{\mG^t - \mV^t}^2,  \label{eq:defn-omega} \\
\Omega_3^t & := \E\normF{\mX^t - \bar\vx^t\vone^\top}^2, \quad 
\Omega_4^t  := \E\normF{\mV^t - \bar\vv^t\vone^\top}^2.   \nonumber
\end{align}
Here, $\Omega_1^t$ and $\Omega_2^t$ denote the compression errors for $\mX^t$ and $\mV^t$ when approximated using the compressed surrogates $\mH^t$ and $\mG^t$ respectively, and $\Omega_3^t$ and $\Omega_4^t$ denote the consensus errors of $\mX^t$ and $\mV^t$. 

\subsection{Convergence in the nonconvex setting}\label{sec:nonconvex}

First, we present the following convergence result of \beer in the nonconvex setting when there is no local variance ($\sigma^2 = 0$), i.e., we can use the local full gradient $\nabla F(\mX^t)$ instead of  $\tilde\nabla_b F(\mX^t)$ in Line~\ref{line:update-v} of Algorithm~\ref{alg:becdo}.

\begin{restatable}[Convergence in the nonconvex setting without local variance]{theorem}{thmnonconvexfull}\label{thm:nonconvex-full}
    Suppose Assumptions \ref{ass:mixing}, and \ref{ass:smooth} hold, and we can compute the local full gradient $\nabla f_i(\vx)$ for any $\vx$. Then there exist absolute constants $c_1 ,c_2 ,c_3 ,c_4 ,c_{\gamma},c_{\eta}>0$, such that if we set $\gamma = c_{\gamma}\alpha\rho$, $\eta = c_{\eta}\gamma\rho^2/L$, then for the Lyapunov function $\Phi_t$ in \eqref{eq:lf}, it holds
    \[\frac{1}{T}\sum_{t=0}^{T-1} \E \norm{\nabla f(\bar\vx^t)}^2 \le \frac{2(\Phi_0 - \Phi_T)}{\eta T}.\]
\end{restatable}
Theorem~\ref{thm:nonconvex-full} shows that \beer converges at a rate of $O(1/T)$ when there is no local variance ($\sigma^2 = 0$), which is faster than the $O(1/T^{2/3})$ rate by \algname{CHOCO-SGD}~\citep{koloskova2019decentralizeddeep} and \algname{DeepSqueeze}~\citep{tang2019deepsqueeze}, and the $O(1/\sqrt{T})$ rate by \algname{SQuARM-SGD}~\citep{singh2021squarm}; see also Table~\ref{tab:existed-results}. 

More specifically, to achieve $\frac{1}{T}\sum_{t=0}^{T-1} \E \norm{\nabla f(\bar\vx^t)}^2 \leq \epsilon^2$, \beer needs 
$$O\left( \frac{1}{\rho^3\alpha\epsilon^2}\right)$$ 
iterations or communication rounds, where $\rho$ and $\alpha$ are the spectral gap (cf.~\eqref{eq:rho}) and the compression parameter (cf.~\eqref{eq:comp}), respectively. In comparison, the state-of-the-art algorithm \algname{CHOCO-SGD} \citep{koloskova2019decentralizeddeep} converges at a rate of $O((G/\rho^2\alpha T)^{2/3})$, which translates to an iteration complexity of $O\left(\frac{G}{\rho^2\alpha\epsilon^3}\right)$, with $G$ being the bounded gradient parameter, namely, $\E_{\xi_i\sim \gD_i}\norm{\nabla f(\vx, \xi_i)}^2 \le G^2$. Therefore, \beer improves over \algname{CHOCO-SGD} not only in terms of a better dependency on $\epsilon$, but also removing the bounded gradient assumption, which is significant since in practice, $G$ can be excessively large due to data heterogeneity across the clients. 

The dependency on $\alpha$ of \beer is consistent with other compression schemes, such as \algname{CHOCO-SGD}, \algname{DeepSqueeze} and \algname{SQuARM-SGD} for the nonconvex setting, as well as \algname{LEAD}~\citep{liu2020linear} and \algname{EF-C-GT}~\citep{liao2021compressed} for the strongly convex setting.
 
 As for the dependency on $\rho$, \beer is slightly worse than \algname{CHOCO-SGD}, where \algname{CHOCO-SGD} has a dependency of $O(1/\rho^2)$ whereas \beer has a dependency of $O(1/\rho^3)$. This degeneration is also seen in the analysis of uncompressed decentralized algorithms using gradient tracking \citep{sun2020improving,xin2020fast}, where the rate $O(1/\rho^2)$ is worse than the rate of $O(1/\rho)$ for basic decentralized SGD algorithms~\citep{kempe2003gossip,lian2017can} by a factor of $\rho$. In addition, both \beer and \algname{CHOCO-SGD} use small mixing step size $\gamma$ to guarantee convergence, which makes the dependency on $\rho$ worse than their uncompressed counterparts.

\paragraph{Stochastic gradient oracles} \beer also supports the use of stochastic gradient oracles with bounded local variance (Assumption \ref{ass:bounded-variance}). More specifically, we have the following theorem, which generalizes Theorem \ref{thm:nonconvex-full}.

\begin{restatable}[Convergence in the nonconvex setting]{theorem}{thmnonconvex}\label{thm:nonconvex}
    Suppose Assumptions \ref{ass:mixing}, \ref{ass:smooth} and \ref{ass:bounded-variance} hold. Then there exist absolute constants $c_1 ,c_2 ,c_3 ,c_4 ,c_{\gamma},c_{\eta}>0$, such that if we set $\gamma = c_{\gamma}\alpha\rho$, $\eta = c_{\eta}\gamma\rho^2/L$, then for the Lyapunov function $\Phi_t$ in \eqref{eq:lf}, it holds
    \[\frac{1}{T}\sum_{t=0}^{T-1} \E \norm{\nabla f(\bar\vx^t)}^2 \le \frac{2(\Phi_0 - \Phi_T)}{\eta T} + \frac{36c_4 \sigma^2}{c_{\gamma} b\alpha L}.\]
\end{restatable}

In the presence of local variance, the squared gradient norm of \beer has an additional term that scales on the order of $O\left(\frac{\sigma^2}{\alpha b} \right)$ (ignoring other parameters). By choosing a sufficiently large minibatch size $b$, i.e. $b \ge O\left(\frac{\sigma^2}{\alpha \epsilon^2} \right)$, \beer maintains the iteration complexity
    \[O\left(\frac{1}{\rho^3\alpha\epsilon^2}\right)\]
to reach $\frac{1}{T}\sum_{t=0}^{T-1} \E \norm{\nabla f(\bar\vx^t)}^2 \leq \epsilon^2$, without the bounded gradient assumption, thus inheriting similar advantages over \algname{CHOCO-SGD} as discussed earlier. In terms of local computation, the gradient oracle complexity on a single client of \beer is
 \[O\left(\frac{1}{\rho^3\alpha\epsilon^2} + \frac{\sigma^2}{\rho^3\alpha^2 \epsilon^4}\right).\]

While our focus is on communication efficiency, to gain more insights, Table \ref{tab:detailed-comparison} summarizes both the communication rounds and the gradient complexity for different decentralized stochastic methods. While \beer does not require the bounded gradient assumption, it may lead to a worse gradient complexity in the data homogeneous setting due to the use of large minibatch size. Fortunately, this only impacts the local computation cost, and does not exacerbate the communication complexity, which is often the bottleneck. It is of great interest to further refine the design and analysis of \beer in terms of the gradient complexity.

\begin{table*}[!t]
\renewcommand{\arraystretch}{1.5}
    \centering
 
    \begin{tabular}{|c|c|c| }
    \hline
    \bf Algorithm & \makecell{\bf Communication rounds}  & \makecell{\bf Gradient complexity}    \\
    \hline
        \makecell{\algname{SQuARM-SGD}  \citep{singh2021squarm}} & $O\left(\frac{nG^2}{\epsilon^2}+\frac{\sigma^2}{bn\epsilon^4}\right)$ & $O\left(\frac{\sigma^2}{n\epsilon^4}+\frac{nG^2}{\epsilon^2}\right)$    \\ \hline
        \makecell{\algname{DeepSqueeze}   \citep{tang2019deepsqueeze}} & $O\left(\frac{G}{\epsilon^3}+\frac{\sigma^2}{bn\epsilon^4}\right)$ & $O\left(\frac{\sigma^2}{n\epsilon^4}+\frac{G}{\epsilon^3}\right)$       \\ \hline
        \makecell{\algname{CHOCO-SGD}    \citep{koloskova2019decentralizeddeep}} & $O\left(\frac{G}{\epsilon^3}+\frac{\sigma^2}{bn\epsilon^4}\right)$ & $O\left(\frac{\sigma^2}{n\epsilon^4}+\frac{G}{\epsilon^3}\right)$  \\ \hline
        \cellcolor{bgcolor} \beer (Algorithm \ref{alg:becdo})   \ \ \ \  & \cellcolor{bgcolor} $O\left(\frac{1}{\epsilon^2}\right)$ & \cellcolor{bgcolor} $O\left(\frac{\sigma^2}{\epsilon^4}+\frac{1}{\epsilon^2}\right)$     \\ \hline
    \end{tabular}
     \caption{A more detailed comparison of the communication complexity and the gradient complexity with existing decentralized stochastic methods in the nonconvex setting to reach $\epsilon$-accuracy. Here, $G$ again measures the bounded gradient or bounded dissimilarity assumption,
 $\sigma^2$ and $b$ denote the gradient variance and batch size respectively. We omit the dependency on the compression ratio and the network topology parameter for brevity.}
     \label{tab:detailed-comparison}
\end{table*}

\subsection{Linear convergence with PL condition}\label{sec:pl}

Now, we show that the convergence of \beer can be strengthened to a linear rate with the addition of the PL condition (Assumption \ref{ass:pl}). Similar to the nonconvex setting, we first show the convergence result without local gradient variance ($\sigma^2 = 0$).

\begin{restatable}[Convergence under the PL condition without local variance]{theorem}{thmplfull}\label{thm:pl-full}
    Suppose Assumptions \ref{ass:mixing}, \ref{ass:smooth}, and \ref{ass:pl} hold, and we can compute the local full gradient $\nabla f_i(\vx)$ for any $\vx$. Then there exist constants $c_1 ,c_2 ,c_3 ,c_4 ,c_{\gamma},c_{\eta}>0$, such that if we set $\gamma = c_{\gamma}\alpha\rho$, $\eta = c_{\eta}\gamma\rho^2/L$, then for the Lyapunov function $\Phi_t$ in \eqref{eq:lf}, it holds
    \[\Phi_T \le (1-\mu\eta)^T \Phi_0.\]
\end{restatable}

Theorem \ref{thm:pl-full} demonstrates that under the PL condition, \beer converges linearly to the global optimum $f^*$, where it finds an $\epsilon$-optimal solution in $O\left( \frac{L}{\mu\rho^3\alpha} \log\left(\frac{1}{\epsilon}\right)  \right)$ iterations.

\paragraph{Stochastic gradient oracles} Under the PL condition, \beer also supports the use of stochastic gradient oracles with bounded local variance (Assumption \ref{ass:bounded-variance}). The following theorem shows that \beer linearly converges to a neighborhood of size $O\left(\frac{\sigma^2}{\alpha b}\right)$ around the global optimum.

\begin{restatable}[Convergence under PL condition]{theorem}{thmpl}\label{thm:pl}
    Suppose Assumptions \ref{ass:mixing}, \ref{ass:smooth}, \ref{ass:bounded-variance}, and \ref{ass:pl} hold. Then there exist absolute constants $c_1 ,c_2 ,c_3 ,c_4 ,c_{\gamma},c_{\eta}>0$, such that if we set $\gamma = c_{\gamma}\alpha\rho$, $\eta = c_{\eta}\gamma\rho^2/L$, then for the Lyapunov function $\Phi_t$ in \eqref{eq:lf}, it holds
    \[\Phi_T \le (1-\mu\eta)^T \Phi_0 + \frac{36c_4  \sigma^2}{c_{\gamma} L b\alpha}.\]
\end{restatable}

\subsection{Proof sketch}\label{sec:proof-sketch}

We now provide a proof sketch of Theorem \ref{thm:nonconvex-full}, which establishes the $O(1/T)$ rate of \beer in the nonconvex setting using full gradient, highlighting the technical reason of the rate improvement of \beer over \algname{CHOCO-SGD}.

Recalling the quantities $\Omega_1^t$ to $\Omega_4^t$ from (\ref{eq:defn-omega}), which capture the approximation errors using compression and the consensus errors of $\mX^t$ and $\mV^t$, we would like to control these errors by obtaining inequalities of the form:
\[\Omega_i^{t+1} \le (1-a_i) \Omega_i^t + b_i, \quad \forall i \in \{1,2,3,4\},\]
where $0<a_i <1$ denotes the size of the contraction, and $b_i>0$ wraps together other terms, which may depend on $\Omega_j^t$ for $j\neq i$, as well as the expected squared gradient norm of $\bar\vv^t$, i.e.,
\begin{equation}\label{eq:Omega_v}
\Omega_5^t = \E\norm{\bar\vv^t}^2.
\end{equation}
Then, by choosing the Lyapunov function properly (cf.~\eqref{eq:lf}), we can show that the Lyapunov function actually descends, and small manipulations lead to the claimed convergence results in Theorem \ref{thm:nonconvex-full}.
 
We now explain briefly how gradient tracking helps in \beer. Note that \algname{CHOCO-SGD} also has the control variable $\mH^t$ for the model $\mX^t$, therefore in its analysis, it also deals with the quantities $\Omega_1^t$ and $\Omega_3^t$. However, \algname{CHOCO-SGD} also needs to bound the term $\normF{\mV^t}^2$, where $\mV^t = \nabla F(\mX^t)$ for \algname{CHOCO-SGD} when using full gradients. Thus, \algname{CHOCO-SGD} needs to assume the bounded gradient assumption and only obtain a slower $O(1/T^{2/3})$ convergence rate. In contrast, \beer deals with the term $\normF{\mV^t}^2$ by decomposing it using  Young's inequality, leading to
\[\normF{\mV^t}^2 \le (1+\beta)\Omega_4^t + (1+1/\beta)\Omega_5^t\]
for some $\beta > 0$. Here, $\Omega_4^t$ can be controlled via the \emph{gradient tracking} technique (see Line \ref{line:update-v} in Algorithm \ref{alg:becdo}) {\em without} the bounded gradient assumption, and $\Omega_5^t$ can be handled using the smoothness assumption (Assumption~\ref{ass:smooth}).

\section{Numerical Experiments}
\label{section:numerical_experiments}
This section presents numerical experiments on real-world datasets to showcase \beer's superior ability to handle data heterogeneity across the clients,
by running each experiment on unshuffled datasets and comparing the performances with the state-of-the-art baseline algorithms both with and without communication compression. The code can be accessed at:
\begin{center}
\url{https://github.com/liboyue/beer}.
\end{center}

We run experiments on two nonconvex problems: logistic regression with a nonconvex regularizer~\citep{wang2018spiderboost} on the \dataset{a9a} dataset~\citep{chang2011libsvm}, and training 1-hidden layer neural network on the \dataset{MNIST} dataset \citep{lecun1995learning}.  
For logistic regression with a nonconvex regularizer, following~\citet{wang2018spiderboost}, the objective function over a datum $(\va,b)$ is defined as
    \[f(\vx; (\va,b)) =   \log\left(1+\exp(-b \va^\top\vx )\right) + \alpha\sum \limits_{j=1}^d\frac{x_j^2}{1+x_j^2},\]
where the last term is the nonconvex regularizer and the regularization parameter is set to $\alpha=0.05$. 

For 1-hidden layer neural network training,
we use $32$ hidden neurons, sigmoid activation functions and cross-entropy loss.
The objective function over a datum $(\va,b)$ is defined as 
$$f(\vx; (\va,b)) = \ell(\mathsf{softmax}(\bm W_2 ~\mathsf{sigmoid}( \bm W_1 \va + \bm c_1) + \bm c_2), b),$$
where $\ell(\cdot, \cdot)$ denotes the cross-entropy loss, the optimization variable is collectively denoted by $\vx = \text{vec}(\bm W_1, \bm c_1, \bm W_2, \bm c_2)$, where the dimensions of the network parameters $\bm W_1$, $\bm c_1$, $\bm W_2$, $\bm c_2$ are $64 \times 784$, $64 \times 1$, $10 \times 64$, and $10 \times 1$, respectively.

For both experiments, we split the \emph{unshuffled} datasets evenly to $10$ clients that are connected by a ring topology. By using unshuffled data, we can simulate the scenario with high data heterogeneity across clients.
Approximately, for the \dataset{a9a} dataset, 5 clients receive data with label $1$ and others receive data with label $0$; for the \dataset{MNIST} dataset, client $i$ receives data with label $i$. We use the FDLA matrix \citep{xiao2004fast} as the mixing matrix to perform weighted information aggregation to accelerate convergence.

\begin{figure*}[tb]
\begin{center}
\begin{tabular}{cc}
\includegraphics[width=0.4\textwidth]{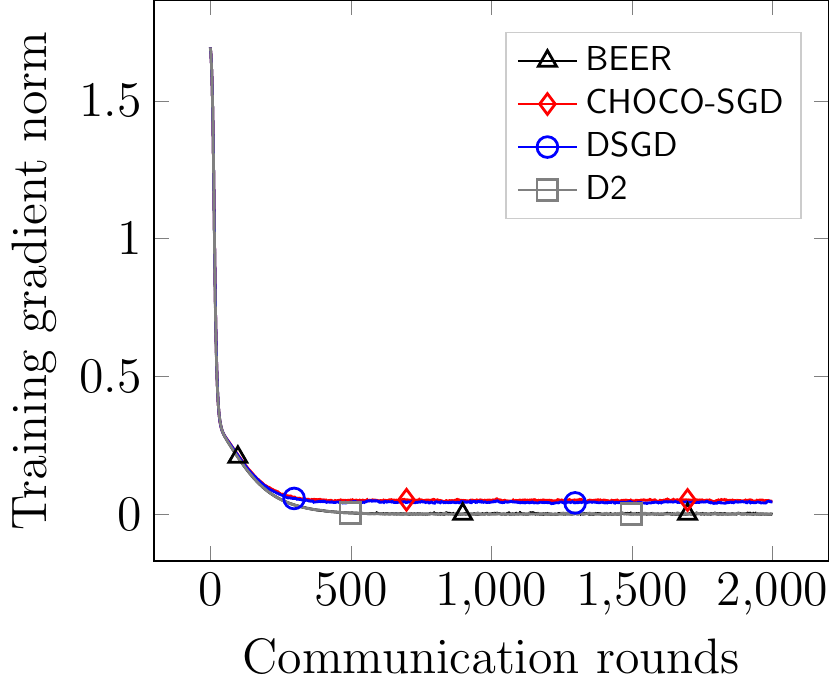} & \includegraphics[width=0.4\textwidth]{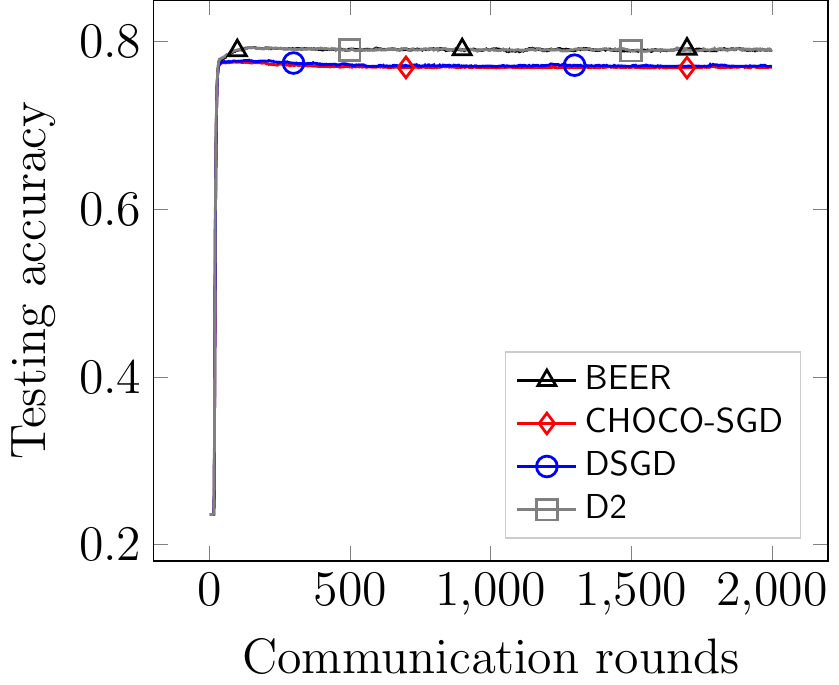} \\
\vspace{0.05in}\\
\includegraphics[width=0.4\textwidth]{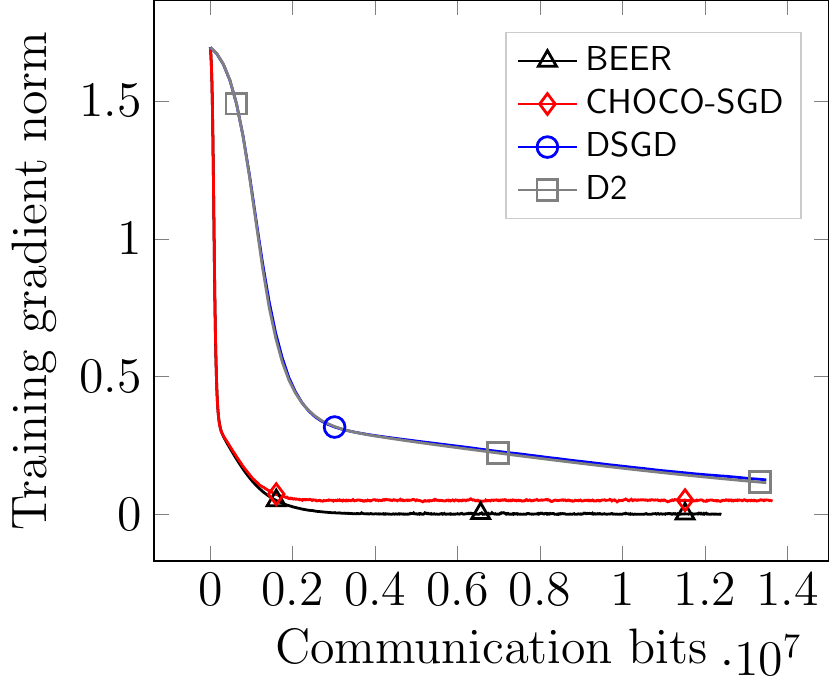} & \includegraphics[width=0.4\textwidth]{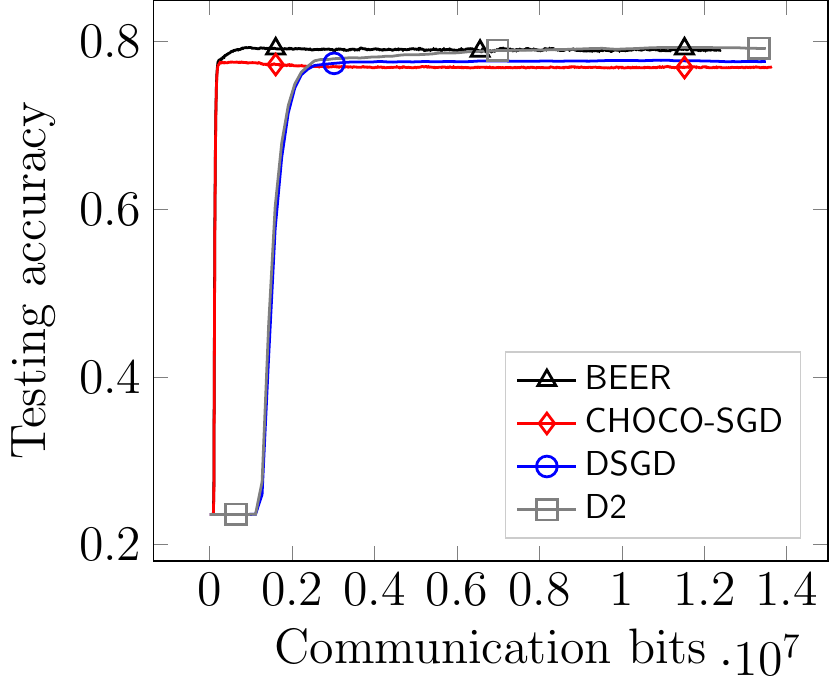}
\end{tabular}
\end{center}
\caption{The training gradient norm and testing accuracy against communication rounds (top two panels) and communication bits (bottom two panels) for logistic regression with nonconvex regularization on unshuffled \dataset{a9a} dataset. Both \beer and \algname{CHOCO-SGD}  employ the biased $\text{gsgd}_b$ compression \citep{alistarh2017qsgd} with $b=5$. \label{fig:logistic_regression}}
\end{figure*}

\begin{figure*}[tb]
\begin{center}
\begin{tabular}{cc}
\includegraphics[width=0.385\textwidth]{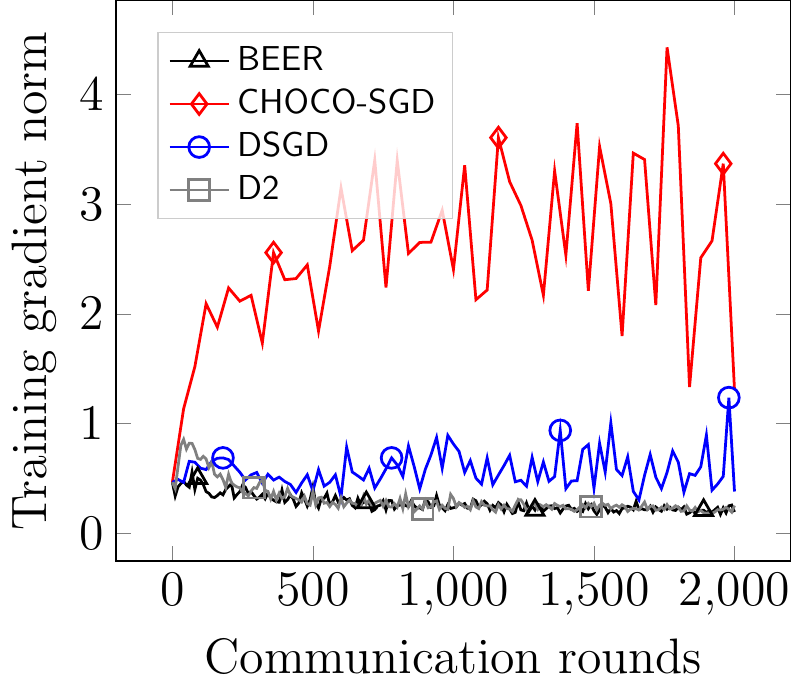} & \includegraphics[width=0.4\textwidth]{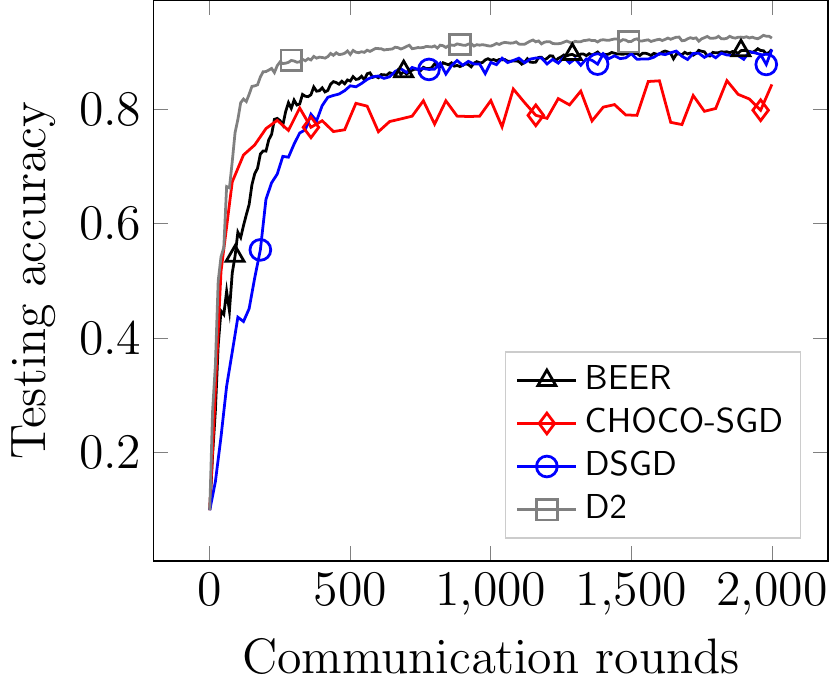} \\
\vspace{0.05in}\\
\includegraphics[width=0.385\textwidth]{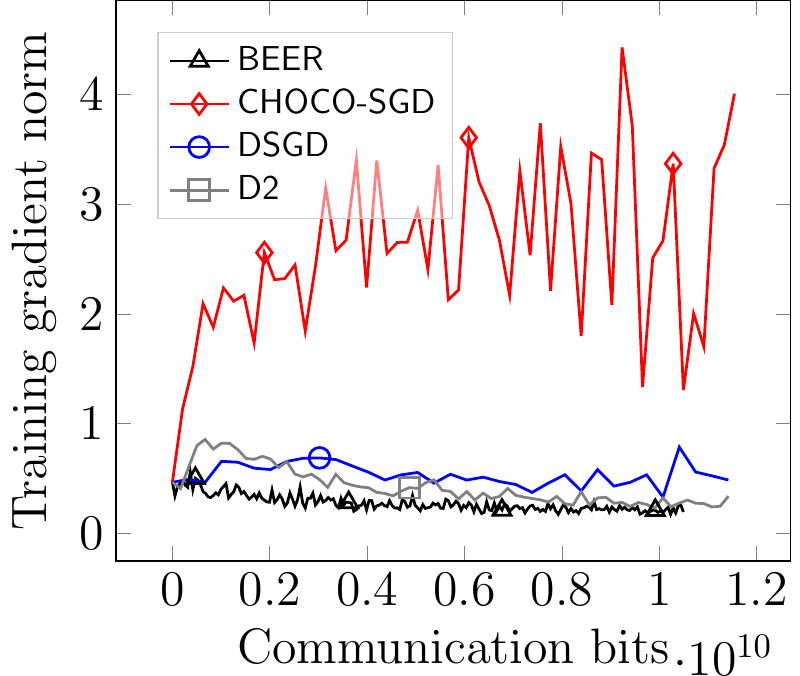} & \includegraphics[width=0.4\textwidth]{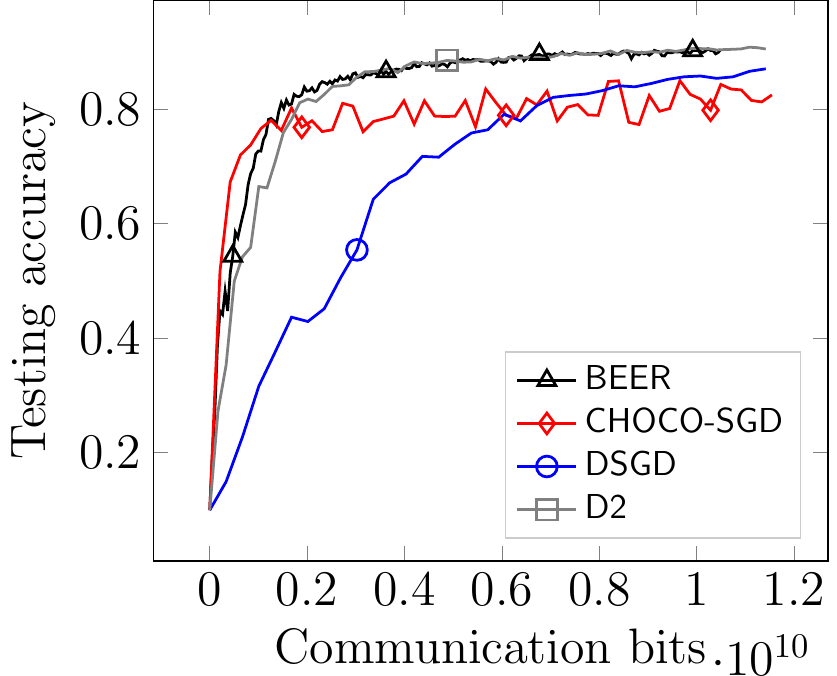}
\end{tabular}
\end{center}
\caption{The training gradient norm and testing accuracy against communication rounds (top two panels) and communication bits (bottom two panels) for classification on unshuffled \dataset{MNIST} dataset using a 1-hidden-layer neural network. Both \beer and \algname{CHOCO-SGD} employ the biased $\text{gsgd}_b$ compression \citep{alistarh2017qsgd} with $b=20$. \label{fig:nn}}
\end{figure*}

\subsection{Comparisons with state-of-the-art algorithms} 

We compare \beer with 1) \algname{CHOCO-SGD} \citep{koloskova2019decentralized}, which is the state-of-the-art nonconvex decentralized optimizing algorithm using communication compression, and 2)  \algname{DSGD}~\citep{lian2017can} and \algname{$D^2$}~\citep{tang2018d}, which are decentralized optimization algorithms without compression. All algorithms are initialized in the same experiment by the same initial point. Moreover, we use the same best-tuned learning rate $\eta = 0.1$, batch size $b = 100$, and biased compression operator ($\text{gsgd}_b$)~\citep{alistarh2017qsgd} for \beer and \algname{CHOCO-SGD} on both experiments.

\Cref{fig:logistic_regression} and \Cref{fig:nn} plot the training gradient norm and testing accuracy against communication rounds and communication bits for logistic regression with nonconvex regularization and 1-hidden-layer neural network training, respectively.

In the nonconvex logistic regression experiment (cf.~\Cref{fig:logistic_regression}), the algorithms with communication compression (\beer and \algname{CHOCO-SGD}) converge faster than the uncompressed algorithms (\algname{DSGD} and \algname{$D^2$}) in terms of the communication bits. However, \algname{CHOCO-SGD} fails to converge to a small gradient norm solution since it cannot tolerate a high level of data dissimilarity across different clients, and its performance is not comparable to \algname{$D^2$}. In contrast, \beer can converge to a point with a relatively smaller gradient norm, which is comparable to \algname{$D^2$}. The performance of testing accuracy is similar to that of the training gradient norm, where \beer achieves the best testing accuracy and also learns faster than the uncompressed algorithms.

Turning to the neural network experiment (cf.~\Cref{fig:nn}),
in terms of the final training gradient norm,
\beer converges to a solution comparable to \algname{$D^2$}, but at a lower communication cost,
while \algname{CHOCO-SGD} and \algname{DSGD} cannot converge due to the data heterogeneity.
In terms of testing accuracy, \beer and \algname{$D^2$} have very similar performance, and outperform \algname{CHOCO-SGD} and \algname{DSGD}.

In summary, \beer has much better performance in terms of communication efficiency than \algname{CHOCO-SGD} in heterogeneous data scenario, which corroborates our theory.
\beer also performs similarly or even better than the uncompressed baseline algorithm \algname{$D^2$}, and much better than \algname{DSGD}. In addition, by leveraging different communication compression schemes, \beer allows more flexible trade-offs between communication and computation, making it an appealing choice in practice.

\subsection{Impacts of network topology and compression schemes on \beer}

We further investigate the impact of network topology and compression schemes on the performance of \beer.
We follow the same setup as above, and run logistic regression with nonconvex regularization ($\alpha = 0.05$) on the unshuffled \dataset{a9a} dataset, by splitting it evenly to $40$ agents. All experiments use the same best-tuned step size $\eta = 0.5$, batch size $b=100$, and $\gamma = 0.7$, to guarantee fast convergence.

\paragraph{Impacts of network topology} Figure \ref{fig:topologies} shows the training gradient norm and testing accuracy of \beer with respect to the communication rounds over different network topologies using the $\text{gsgd}_5$ compression \citep{alistarh2017qsgd}: ring topology ($\rho = 0.022$),
star topology ($\rho = 0.049$),
grid topology ($\rho = 0.063$), 
and Erd\"{o}s-R\`{e}nyi topology with connectivity probability $p = 0.5$ and $p = 0.2$ ($\rho = 0.51$ and $\rho = 0.77$, respectively). Despite the huge difference of the spectral gaps $\rho$ of different topologies, \beer can use nearly the same hyper-parameters to obtain similar performance. The experiment results complement our theoretical analysis and show that \beer may converge way better in practice despite its cubic dependency of $1/\rho$ in Theorem \ref{thm:nonconvex}.

\begin{figure}[tb]
\begin{center}
\begin{tabular}{cc}
\includegraphics[width=0.4\textwidth]{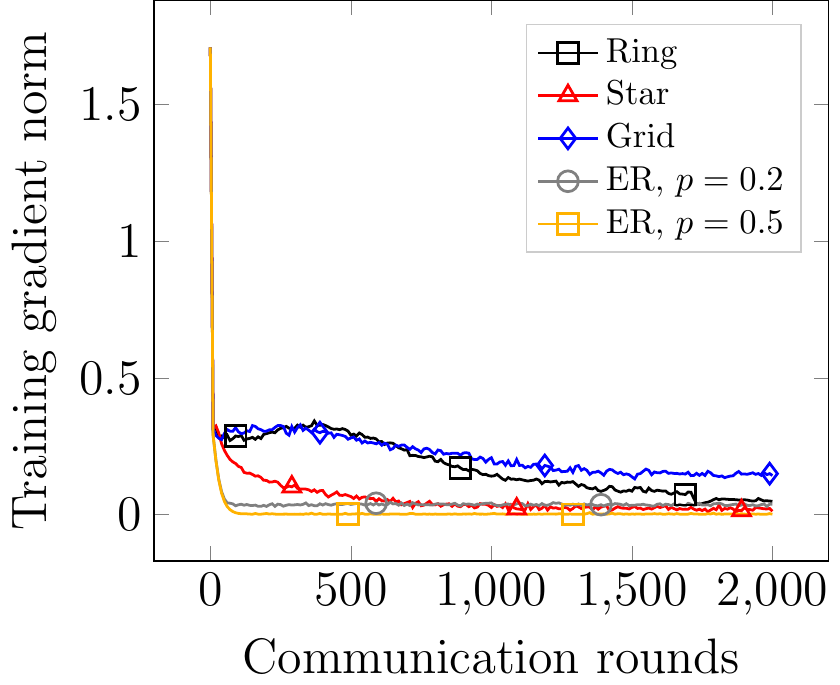} & \includegraphics[width=0.4\textwidth]{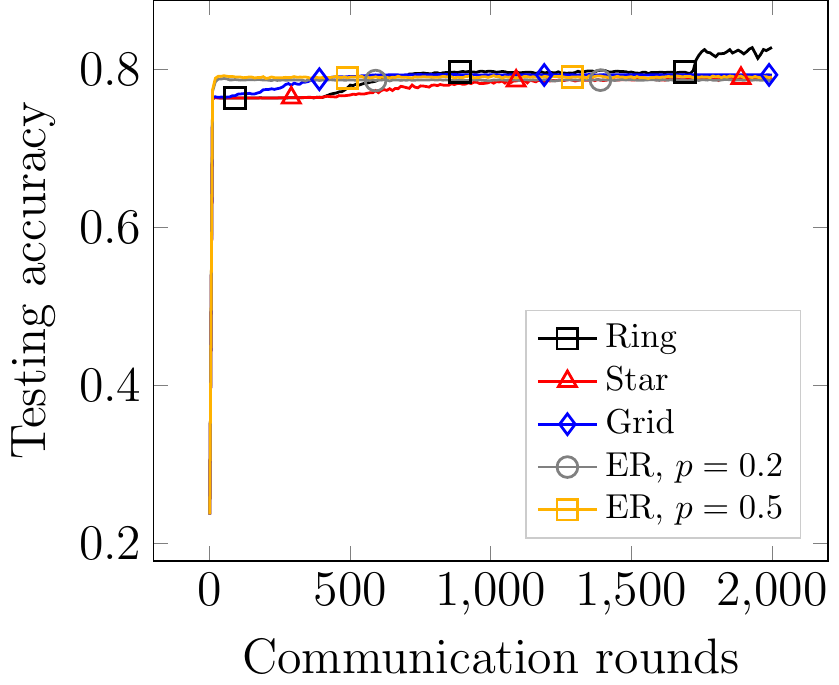} 
\end{tabular}
\end{center}
\caption{The training gradient norm and testing accuracy against communication rounds for \beer using the biased $\text{gsgd}_b$ compression \citep{alistarh2017qsgd} with $b=5$
for logistic regression with nonconvex regularization on unshuffled \dataset{a9a} dataset. \label{fig:topologies}}
\end{figure}

\paragraph{Impacts of compression schemes} \Cref{fig:compressions} shows the training gradient norm and testing accuracy of \beer with respect to the communication rounds and communication bits on a ring topology using different compression schemes: no compression, $\text{gsgd}_{5}$ and $\text{top}_{10}$ (see Appendix \ref{sec:comp} for their formal definitions).
The parameters within the compression operators are chosen such that \beer with different compression operators transfer similar amount of bits per communication round.
All experiments use the same best-tuned step size $\eta = 0.5$, batch size $b=100$ and $\gamma = 0.7$, except that we use $\eta = 0.005$ and $\gamma = 0.8$ for $\text{top}_{10}$ compression. In terms of communication bits, it can be seen that
the use of compression operators improves over the uncompressed baseline,
in the sense that,
all algorithms with compression converge to a solution with lower gradient norm and higher testing accuracy at a lower communication cost.
In terms of communication rounds and testing accuracy, different compression operators can lead to significantly behaviors. For example, \beer with $\text{gsgd}_5$ compression operator converges faster than \beer without compression, but \beer with the $\text{top}_{10}$ compression operator converges slower than \beer without compression.
Among all experiments, \beer with $\text{gsgd}_5$ reaches the highest final testing accuracy,
while behaves similar to \beer without compression in terms of communication rounds,
which implies that $\text{gsgd}_b$ may be a more practical compression operator, at least more suitable for \beer.


\begin{figure*}[!htb]
\vspace{5mm}
\begin{center}
\begin{tabular}{cc}
\includegraphics[width=0.4\textwidth]{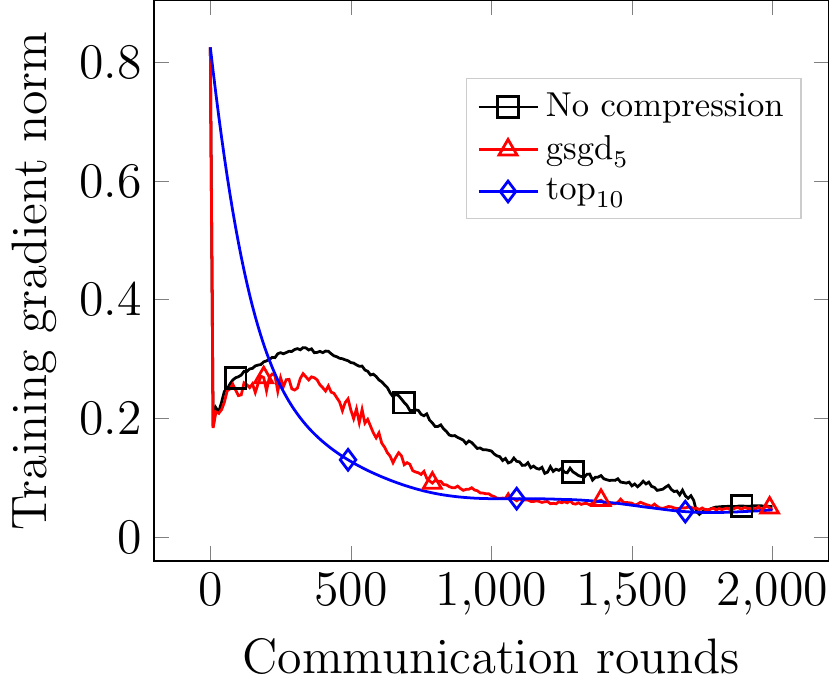} & \includegraphics[width=0.4\textwidth]{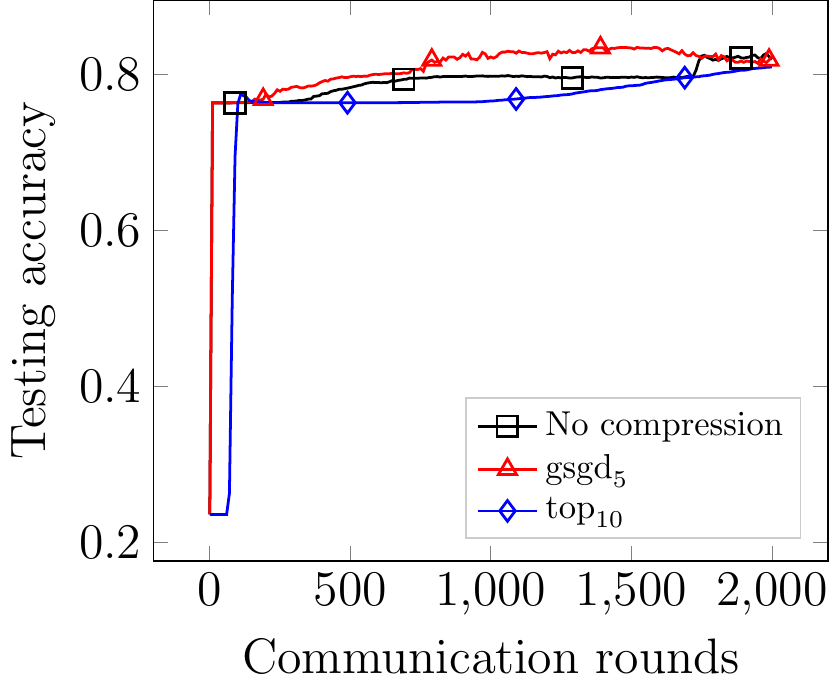} \\
\vspace{0.05in}\\
\includegraphics[width=0.4\textwidth]{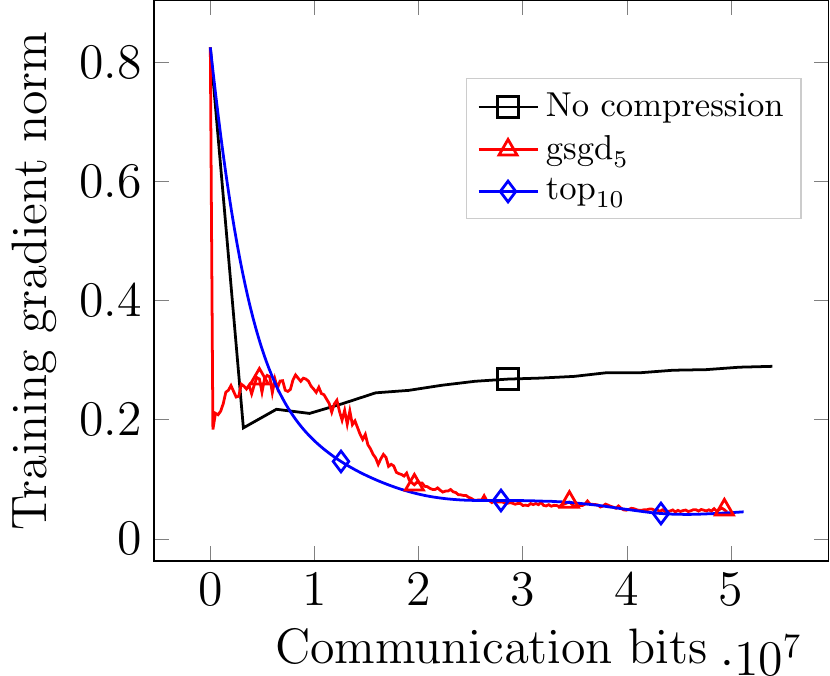} & \includegraphics[width=0.4\textwidth]{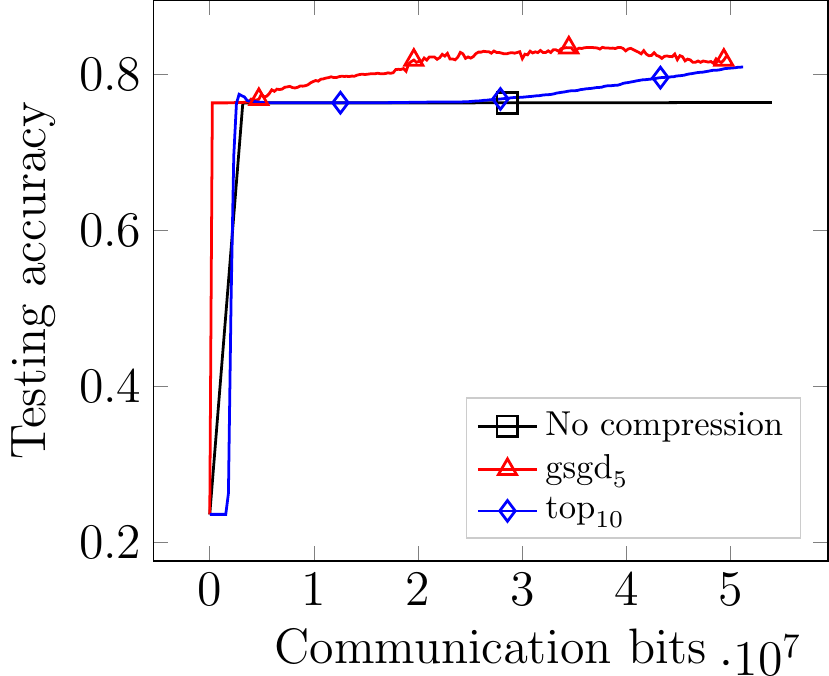}
\end{tabular}
\end{center}
\caption{The training gradient norm and testing accuracy against communication rounds (top two panels) and communication bits (bottom two panels) for \beer using different compression schemes for logistic regression with nonconvex regularization on unshuffled \dataset{a9a} dataset.
\label{fig:compressions}}
\end{figure*}

\subsection{Convolutional network network training}

We further compare the performance of \beer and \algname{CHOCO-SGD} on training a convolutional neural network using the unshuffled \dataset{MNIST} dataset.
Specifically, the network is consist of three modules:
the first module is a 2-d convolution layer (1 input channel, 16 output channels, kernel size 5, stride 1 and padding 2) followed by 2-d batch normalization, ReLU activation and 2-d max pooling (kernel size 2 and stride 2);
the second module is the same as the first module, except the convolution layer has 16 input channels and 32 output channels;
the last module is a fully-connected layer with 1568 inputs and 10 outputs. We adopt the standard cross-entropy loss, and simply average each agent's model with its neighbors. \Cref{fig:conv_nn} shows the testing gradient norm and accuracy against the communication bits. It can be seen that \beer outperforms \algname{CHOCO-SGD}
in terms of both testing gradient norm and testing accuracy.
Both algorithms converge fast initially,
however, due to to the extreme data heterogeneity,
their convergence speeds significantly degenerate after a short time.
\beer keeps improving the objective when \algname{CHOCO-SGD} hits its error floor,
which highlights \beer's advantage to deal with data heterogeneity.

\begin{figure*}[!htb]
\begin{center}
\centerline{\includegraphics[width=0.8\textwidth]{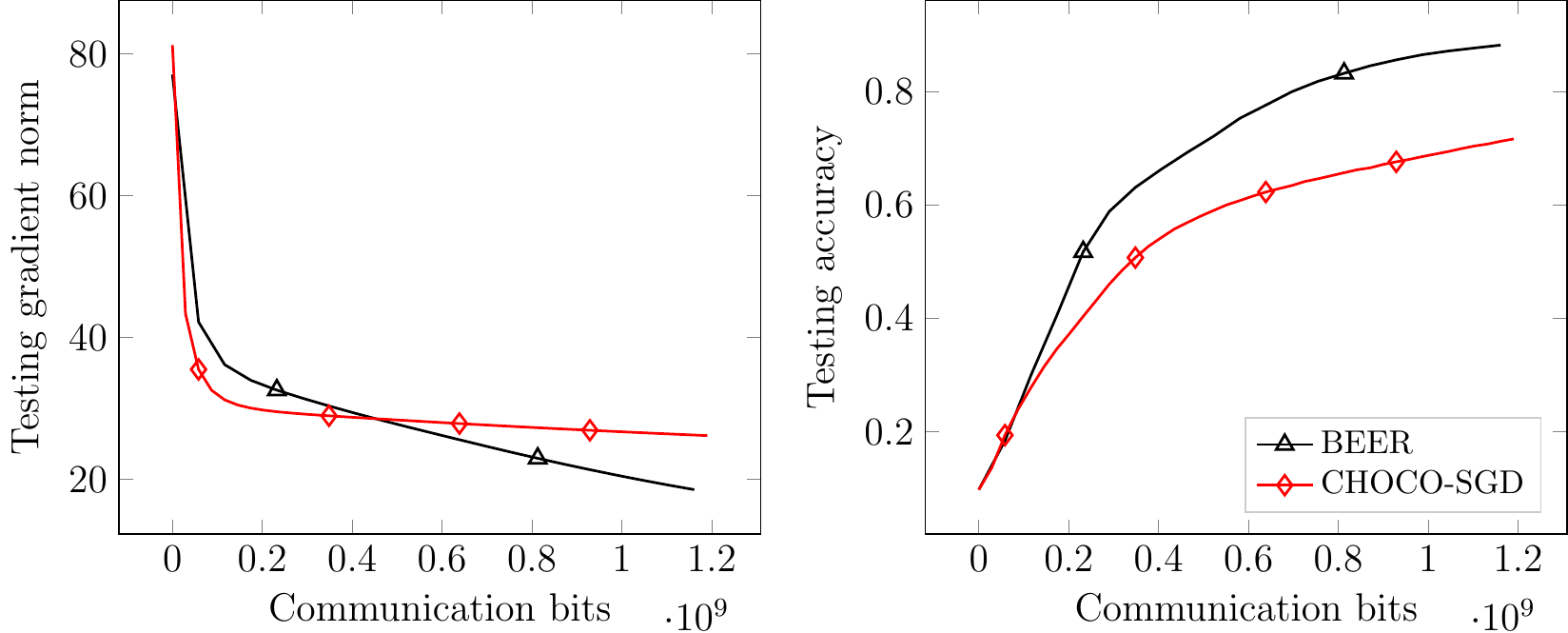}}
\caption{The testing gradient norm and testing accuracy against communication bits on unshuffled \dataset{MNIST} dataset using a 3-layer convolutional neural network. Both \beer and \algname{CHOCO-SGD} employ the biased $\text{gsgd}_b$ compression \citep{alistarh2017qsgd} with $b=5$. \label{fig:conv_nn}}
\end{center}
\end{figure*}

\section{Conclusion}

This paper presents \beer, which achieves a faster $O(1/T)$ convergence rate for decentralized nonconvex optimization with communication compression, {\em without} imposing the bounded dissimilarity or bounded gradient assumptions. In addition, a faster linear rate of convergence is established for \beer under the PL condition. Numerical experiments are provided to corroborate our theory on the advantage of \beer in the data heterogeneous scenario. An interesting direction of future work is to investigate the lower bounds for decentralized (nonconvex) optimization with communication compression. In addition, improving the dependency of \beer with the network topology parameter $\rho$, possibly leveraging the analysis in \citet{koloskova2021improved}, is of interest.

\section*{Acknowledgements}
The work of H. Zhao is supported in part by NSF, ONR, Simons Foundation, DARPA and SRC through awards to S. Arora. The work of B. Li, Z. Li and Y. Chi is supported in part by ONR N00014-19-1-2404, by AFRL under FA8750-20-2-0504, and by NSF under CCF-1901199, CCF-2007911 and CNS-2148212.
B. Li is also gratefully supported by Wei Shen and Xuehong Zhang Presidential Fellowship at Carnegie Mellon University.
The work of P. Richt\'arik is supported by KAUST Baseline Research Fund.
 
\bibliographystyle{apalike}
\bibliography{ref.bib}

\appendix

\section{Examples of Compression Operators}\label{sec:comp}
 
We provide some examples of compression operators satisfying Definition \ref{def:comp} that are used in our experiments.

\paragraph{$\text{gsgd}_b$~\citep{alistarh2017qsgd}} $\text{gsgd}_b: \R^d\to\R^d$ ($b > 1$), or random dithering, is a compression operator satisfying the following formula
\[\text{gsgd}_b(\vx) := \frac{\norm{\vx}}{\tau}\cdot \sign(\vx) \cdot 2^{-(b-1)} \cdot \left\lfloor \frac{2^{(b-1)} | \vx|}{\norm{\vx}} + \vu \right\rfloor,\]
where $\tau = 1 + \min\left\{\frac{d}{2^{2(b-1)}},\frac{\sqrt{d}}{2^{(b-1)}}\right\}$, and $\vu$ is the random dithering vector uniformly sampled from $[0,1]^d$. It follows that $\text{gsgd}_b$ satisfies Definition \ref{def:comp} with $\alpha = 1 / \tau$.


\paragraph{$\text{top}_k$~\citep{alistarh2018convergence,stich2018sparsified}} $\text{top}_k: \R^d\to\R^d$ is a compression operator satisfying the following formula
\[\text{top}_k(\vx) := \vx\odot \vu(\vx),\]
where $\vu(\vx)\in\{0,1\}^d$ that satisfies $\norm{\vu}_1 = k$ and $\vu_i = 1$ for all $i\in \gI$ such that $|\vx_i| \ge |\vx_j|$ for any $i\in\gI$ and $j\in [d]\setminus \gI$. In words, $\text{top}_k$ keeps the coordinates of $\vx$ with the $k$ largest absolute values, and sets the other coordinates to $0$. It follows that $\text{top}_k(\vx)$ satisfies Definition \ref{def:comp} with $\alpha = k/d$.

\section{Proof of Main Theorems}

\subsection{Technical preparation}
We first recall some classical inequalities that helps our derivation.
\begin{proposition}
    Let $\{\vv_1,\dots,\vv_{\tau}\}$ be a set of $\tau$ vectors in $\R^d$. Then, $\forall  \beta > 0$, we have
    \begin{align}
         \langle \vv_i,\vv_j\rangle & \le \frac{\beta}{2}\|\vu\|^2 + \frac{1}{2\beta}\|\vv\|^2, \label{eq:cauchy}\\
        \|\vv_i + \vv_j\|^2 & \le (1 + \beta)\|\vv_i\|^2 + \left(1+\frac{1}{\beta}\right)\|\vv_j\|^2,  \label{eq:rti-1}\\
       \left\|\sum_{i=1}^{\tau}\vv_i\right\|^2  & \le\tau\sum_{i=1}^{\tau}\|\vv_i\|^2.\label{eq:rti-2}
    \end{align}
Here, (\ref{eq:cauchy}) is referred as the Cauchy-Schwarz inequality, (\ref{eq:rti-1}) and (\ref{eq:rti-2}) are referred as Young's inequality.
 
 \end{proposition}

\paragraph{Additional notation} The following notation will be used throughout our proof:
\begin{align*}
\nabla F(\mX) & := [\nabla f_1(\vx_1), \nabla f_2(\vx_2),\ldots,\nabla f_n(\vx_n)], \qquad \nabla F_b(\mX) := [\nabla f_1(\vx_1), \nabla f_2(\vx_2),\ldots,\nabla f_n(\vx_n)],  \\
\nabla F(\bar\vx) & := [\nabla f_1(\bar\vx), \nabla f_2(\bar\vx), \dots, \nabla f_n(\bar\vx)],\qquad ~~~~\tilde\nabla_b F(\bar\vx) := [\tilde\nabla_b f_1(\bar\vx), \tilde\nabla_b f_2(\bar\vx), \dots, \tilde\nabla_b f_n(\bar\vx)],
\end{align*} 
where $\bar\vx := \frac{1}{n}\mX\vone$ with $\mX = [\vx_1,\vx_2,\dots,\vx_n]$.

\paragraph{Properties of the mixing matrix}
We make note of several useful properties of the mixing matrix in the following lemma. 
\begin{lemma}\label{lem:prop_mixing}
Let $\mW$ be a mixing matrix satisfying Assumption \ref{ass:mixing} and has spectral gap $\rho$, then for any matrix $\mM\in\mathbb{R}^{d\times n}$ and $\bar\vm = \frac{1}{n}\mM\vone$, we have
\begin{equation}
    \normF{\mM \mW - \bar\vm \vone^\top}^2 = \normF{\mM \mW - \bar\vm \vone^\top \mW}^2 \le (1-\rho)\normF{\mM - \bar\vm \vone^\top}^2. \label{eq:lem-omega-3-claim1}
\end{equation}
In addition, for any $\gamma \in (0,1]$, the matrix $\widetilde\mW = \mI + \gamma(\mW - \mI)$ satisfies Assumption \ref{ass:mixing} with a spectral gap at least $\gamma\rho$.
\end{lemma}

\begin{proof}
The first claim follows from the spectral decomposition of $\mW$. Since $\mW$ is a doubly stochastic matrix, the largest absolute eigenvalue of $\mW$ is $1$ and the corresponding eigenvector is $\vone$. Let  $\vv_2,\dots,\vv_n$ be the eigenvectors of $\mW$ corresponding to the remaining eigenvalues. Then, we have
\[\normF{\mM \mW - \bar\vm \vone^\top}^2 = \normF{\mM \mW - \bar\vm \vone^\top \mW}^2 = \sum_{i=1}^{r} \norm{\mW (\vm_i - \bar\vm_i\vone)}^2,\]
where the first equality follows from $\vone^\top\mW = \vone^\top$, $\vm_i$ denotes the transpose of $i$-th row of matrix $\mM$, and $\bar\vm_i$ denotes the average of $\vm_i$. Now we decompose $\vm_i - \bar\vm_i\vone$ using the eigenvectors of $\mW$. Noting that
\[\vone^\top(\vm_i - \bar\vm_i\vone) = \vone^\top \vm_i - \vone^\top \vone\frac{1}{n}\vone^\top\vm_i = 0,\]
and thus we can write 
\[\vm_i - \bar\vm_i\vone = \sum_{j=2}^n c_j \vv_j\]
for some $\{c_j\}_{j=2}^n$. Then, we have
\[\norm{\mW (\vm_i - \bar\vm_i\vone)}^2 = \norm{\mW \sum_{j=2}^n c_j \vv_j}^2 \le (1-\rho)^2 \sum_{j=2}^n c_j^2 \le (1-\rho)\sum_{j=2}^n c_j^2 = (1-\rho)\norm{\vm_i - \bar\vm_i\vone}^2,\]
and we conclude the proof of this claim.  

For the second claim, recall the fact that if $\vv$ is an eigenvector of $\mW$ corresponding to the eigenvalue $\lambda$, then $\vv$ is also an eigenvector of $\widetilde\mW$ with the corresponding eigenvalue $(1-\gamma) + \gamma \lambda$.  This claim follows from simple computation based on this relation. 

\end{proof}

\paragraph{A key consequence of gradient tracking}
Before diving in the proofs of the main theorems, 
we record a key property of gradient tracking. Specifically, we have the following lemma.

\begin{lemma}\label{lem:gt}
    If $\bar\vv^0 = \frac{1}{n}\tilde\nabla_b F(\mX^0)\vone$, then for any $t \ge 1$, we have
    \begin{equation}\label{eq:avg_gradient}
    \bar\vv^t = \frac{1}{n}\tilde\nabla_b F(\mX^t)\vone,
    \end{equation} 
   and
    \begin{equation}\label{eq:avg_model}
    \bar\vx^{t+1} = \bar\vx^t - \frac{\eta}{n} \tilde\nabla_b F(\mX^t)\vone.
    \end{equation}
\end{lemma}

\begin{proof}
We first prove \eqref{eq:avg_gradient} by induction. For the base case ($t=0$), the relation \eqref{eq:avg_gradient} is obviously true by the means of initialization. Now suppose that at the $t$-th iteration, the relation \eqref{eq:avg_gradient} is true, i.e.,
\[\bar\vv^t = \frac{1}{n}\tilde\nabla_b F(\mX^t)\vone,\]
then at the $(t+1)$-th iteration, we have
\begin{align}
    \bar\vv^{t+1} & = \frac{1}{n}\mV^{t+1}\vone \nonumber \\
    & = \frac{1}{n}\mV^t\vone + \frac{1}{n} \gamma \mG^{t} (\mW - \mI)\vone + \frac{1}{n}\left(\tilde\nabla_b F(\mX^{t+1}) - \tilde\nabla_b F(\mX^{t})\right)\vone \label{eq:lem-gt-defn} \\
    & = \frac{1}{n}\mV^t\vone + \frac{1}{n}\left(\tilde\nabla_b F(\mX^{t+1}) - \tilde\nabla_b F(\mX^{t})\right)\vone \nonumber \\
    & = \frac{1}{n}\tilde\nabla_b F(\mX^{t+1})\vone. \nonumber
\end{align}
where (\ref{eq:lem-gt-defn}) follows from the update rule of \beer (cf. Line \ref{line:update-v}), the penultimate line follows from $\mW\vone = \vone$, and the last line follows from the induction hypothesis at the $t$-th iteration. Thus the induction hypothesis is also true at the $(t+1)$-th iteration, and we complete the proof of \eqref{eq:avg_gradient}. 

For \eqref{eq:avg_model}, it follows from the update rule of \beer (cf. Line \ref{line:update-x}) that
\begin{align*}
    \bar\vx^{t+1} & = \bar\vx^{t} + \frac{\gamma}{n}\mH^t(\mW-\mI)\mathbf{1} - \frac{\eta}{n}\mV^t\mathbf{1} \\
    & = \bar\vx^{t} - \eta\bar\vv^t = \bar\vx^t - \frac{\eta}{n} \tilde\nabla_b F(\mX^t)\vone,
\end{align*}
where the second line uses $\mW\vone = \vone$ and \eqref{eq:avg_gradient}.
\end{proof}

\subsection{Recursive relations of main errors} 
 
As mentioned previously, the proof is centered around controlling the following set of errors which we repeat below for convenience (cf.~ (\ref{eq:defn-omega})),
\begin{align}
\text{(compression approximation error:) } \quad  \Omega_1^t & = \E\normF{\mH^t - \mX^t}^2 , \qquad     \Omega_2^t  = \E\normF{\mG^t - \mV^t}^2 ,\nonumber \\
\text{(consensus error:) } \quad   \Omega_3^t & = \E\normF{\mX^t - \bar\vx^t\vone^\top}^2, ~~~~\Omega_4^t  = \E\normF{\mV^t - \bar\vv^t\vone^\top}^2, \nonumber \\
\text{(gradient norm:) }  \quad    \Omega_5^t & = \E\norm{\bar\vv^t}^2.\nonumber 
\end{align}
In particular, we aim to build a set of recursive relations of $\Omega_1^t$ to $\Omega_4^t$, which will be specified in the following lemma.  

\begin{lemma}\label{lem:omega-1}
Suppose Assumptions \ref{ass:mixing}, \ref{ass:smooth} and \ref{ass:bounded-variance} hold, then for any $t \ge 0$, we have
\begin{subequations}\label{eq:omega_re}
\begin{align}
    \Omega_1^{t+1} & \le \left(1-\frac{\alpha}{2} + \frac{6\gamma^2C}{\alpha}\right)\Omega_1^t + 0\cdot\Omega_2^t + \frac{6\gamma^2 C}{\alpha}\Omega_3^t + \frac{6\eta^2}{\alpha}\Omega_4^t + \frac{6n\eta^2}{\alpha}\Omega_5^t, \label{eq:omega-1} \\
\Omega_2^{t+1} & \le \frac{18\gamma^2 C L^2}{\alpha}\Omega_1^t + \left(1-\frac{\alpha}{2} + \frac{6\gamma^2 C}{\alpha}\right)\Omega_2^t + \frac{18\gamma^2 C L^2}{\alpha}\Omega_3^t + \frac{6\gamma^2 C+18L^2\eta^2}{\alpha}\Omega_4^t + \frac{18L^2\eta^2 n}{\alpha}\Omega_5^t + \frac{12 n \sigma^2}{b\alpha}, \label{eq:omega-2} \\
    \Omega_3^{t+1} & \le \frac{6\gamma C}{\rho}\Omega_1^t + 0\cdot \Omega_2^t + \left(1-\frac{\gamma\rho}{2}\right)\Omega_3^t + \frac{6\eta^2}{\gamma \rho}\Omega_4^t + 0\cdot\Omega_5^t, \label{eq:omega-3} \\
 \Omega_4^{t+1} & \le \frac{18\gamma C L^2}{\rho}\Omega_1^t + \frac{6\gamma C}{\rho}\Omega_2^t + \frac{18\gamma CL^2}{\rho}\Omega_3^t + \left(1-\frac{\gamma \rho}{2} + \frac{18L^2\eta^2}{\gamma\rho}\right)\Omega_4^t + \frac{18n\eta^2 L^2}{\gamma\rho}\cdot\Omega_5^t + \frac{12 n \sigma^2}{b\gamma\rho}, \label{eq:omega-4}
\end{align}
\end{subequations}
where 
\begin{equation}\label{eq:C}
C = \norm{\mW - \mI}^2 = \sigma_{\max}(\mW - \mI)^2
\end{equation}
 is the square of the maximum singular value of the matrix $\mW - \mI$. 
\end{lemma}
Note that the eigenvalues of $\mW$ and $\mI$ all lies in $[-1,1]$, and thus clearly $C \le 4$.

\begin{proof}
We will establish the inequalities in \eqref{eq:omega_re} one by one.

\paragraph{Bounding $\Omega_1^t$ in \eqref{eq:omega-1}}
First from the update rule of \beer (cf. Line \ref{line:update-h}), we have
\begin{align}
    \normF{\mH^{t+1} - \mX^{t+1}}^2 & = \normF{\mH^{t} + \gC(\mX^{t+1} - \mH^t) - \mX^{t+1}}^2 \nonumber \\
    & \le (1-\alpha)\normF{\mX^{t+1} - \mH^t}^2 \nonumber \\
    & \le \left(1-\frac{\alpha}{2}\right)\normF{\mX^{t} - \mH^t}^2 + \frac{2}{\alpha}\normF{\mX^{t+1} - \mX^{t}}^2, \label{eq:noodle}
    \end{align}
where the first inequality comes from the definition of compression operators (Definition \ref{def:comp}) and the second inequality comes from Young's inequality. It then boils down to bound $\normF{\mX^{t+1} - \mX^{t}}^2$, for which we have
\begin{align}
    \normF{\mX^{t+1} - \mX^{t}}^2 & = \normF{\gamma \mH^t (\mW - \mI)- \eta \mV^t}^2 \\
    & = \normF{\gamma (\mH^t - \mX^t) (\mW - \mI) + \gamma (\mX^t - \bar\vx^t\vone^\top) (\mW - \mI) - \eta \mV^t}^2  \nonumber \\
    & \le 3\gamma^2 C\normF{\mX^t - \mH^t} + 3\gamma^2C\normF{\mX^t - \bar\vx^t\vone^\top}^2 + 3\eta^2\normF{\mV^t}^2  \nonumber\\
    & = 3\gamma^2 C\normF{\mX^t - \mH^t} + 3\gamma^2C\normF{\mX^t - \bar\vx^t\vone^\top}^2 + 3\eta^2\normF{\mV^t - \bar\vv^t\vone^\top}^2 + 3\eta^2 n \norm{\bar\vv^t}^2, \label{eq:diff_x}
\end{align}
where in the first line we use the update rule of \beer (cf. Line \ref{line:update-x}), in the second line we use the property of the mixing matrix $\vone^\top\mW = \vone^\top$, and in the third line, we apply Young's inequality (cf.~\eqref{eq:rti-2}). In the fourth line, we use $\|\vv  \|^2 = \|\vv - \bar{v}\vone\|^2  + n \bar{v}^2$ for any vector $\vv$ with an average $\bar{v}$. Plugging this back into \eqref{eq:noodle}, we get
\begin{align*}
    \normF{\mH^{t+1} - \mX^{t+1}}^2 & \le \left(1-\frac{\alpha}{2} + \frac{6\gamma^2C}{\alpha}\right)\normF{\mX^{t} - \mH^t}^2 + \frac{6\gamma^2C}{\alpha}\normF{\mX^t - \bar\vx^t\vone^\top}^2 \\
    & \qquad + \frac{6\eta^2}{\alpha}\normF{\mV^t - \bar\vv^t\vone^\top}^2 + \frac{6n\eta^2}{\alpha}\norm{\bar\vv^t}^2.
\end{align*}
Plugging in the definitions of $\Omega_i^t$, we obtain \eqref{eq:omega-1}.

\paragraph{Bounding $\Omega_2^t$ in \eqref{eq:omega-2}}
Similar to the derivation of  \eqref{eq:omega-1}, by applying the update rule of $\mG^t$ in \beer (Line \ref{line:update-g}), the definition of compression operators (Definition \ref{def:comp}), and Young's inequality, we have
\begin{align}
    \normF{\mV^{t+1} - \mG^{t+1}}^2 & = \normF{\mG^{t} + \gC(\mV^{t+1} - \mG^{t}) - \mV^{t+1}}^2  \nonumber \\
    & \le (1-\alpha)\normF{\mG^{t}  - \mV^{t+1}}^2  \nonumber\\
    & \le \left(1-\frac{\alpha}{2}\right)\normF{\mG^{t}  - \mV^{t}}^2 + \frac{2}{\alpha}\normF{\mV^{t+1} - \mV^{t}}^2.\label{eq:rice}
\end{align}
It then boils down to bound $\normF{\mV^{t+1} - \mV^{t}}^2$. By the update rule of \beer (cf. Line \ref{line:update-v}), we have
\begin{align}
    \normF{\mV^{t+1} - \mV^{t}}^2 & = \normF{\gamma\mG^t(\mW - \mI) + (\tilde\nabla_b F(\mX^{t+1}) - \tilde\nabla_b F(\mX^{t}))}^2 \nonumber \\
    & = \normF{\gamma (\mG^t - \mV^t)(\mW - \mI) + \gamma(\mV^t-\bar\vv^t\vone^\top)(\mW - \mI) + (\tilde\nabla_b F(\mX^{t+1}) - \tilde\nabla_b F(\mX^{t})}^2 \nonumber \\
    & \overset{\mathrm{(i)}}{\leq} 3\gamma^2C \normF{\mG^t - \mV^t}^2 + 3\gamma^2 C\normF{\mV^t-\bar\vv^t\vone^\top}^2 + 3\normF{\tilde\nabla_b F(\mX^{t+1}) - \tilde\nabla_b F(\mX^{t})}^2 \nonumber \\
    &  \overset{\mathrm{(ii)}}{\leq}  3\gamma^2C \normF{\mG^t - \mV^t}^2 + 3\gamma^2 C\normF{\mV^t-\bar\vv^t\vone^\top}^2 + 3\normF{\nabla F(\mX^{t+1}) - \nabla F(\mX^{t})}^2 + \frac{6 n \sigma^2}{b} \nonumber \\
    &  \overset{\mathrm{(iii)}}{\leq}  3\gamma^2C \normF{\mG^t - \mV^t}^2 + 3\gamma^2 C\normF{\mV^t-\bar\vv^t\vone^\top}^2 + 3L^2\normF{\mX^{t+1} - \mX^t}^2 + \frac{6 n \sigma^2}{b} \nonumber \\
    & \overset{\mathrm{(iv)}}{\leq}  3\gamma^2C \normF{\mG^t - \mV^t}^2 + (3\gamma^2 C + 9L^2\eta^2)\normF{\mV^t-\bar\vv^t\vone^\top}^2 \nonumber \\
    & \qquad + 9\gamma^2CL^2\normF{\mX^t-\mH^t}^2 + 9\gamma^2CL^2\normF{\mX^t-\bar\vx^t\vone^\top}^2 + 9L^2\eta^2n\norm{\bar\vv^t}^2 + \frac{6n \sigma^2}{b}, \nonumber 
\end{align}
where (i) comes from Young's inequality (cf.~\eqref{eq:rti-2}) and basic facts of matrix norm (cf.~\eqref{eq:C}), (ii) comes from the bounded variance assumption (Assumption \ref{ass:bounded-variance}), (iii) comes from the smoothness assumption (Assumption \ref{ass:smooth}), and (iv) follows from \eqref{eq:diff_x}.
Combining the above inequality with \eqref{eq:rice}, we have
\begin{align*}
    \normF{\mV^{t+1} - \mG^{t+1}}^2 & \le \left(1-\frac{\alpha}{2}\right)\normF{\mG^{t}  - \mV^{t}}^2 + \frac{2}{\alpha}\normF{\mV^{t+1} - \mV^{t}}^2 \\
    & \le \left(1-\frac{\alpha}{2} + \frac{6\gamma^2 C}{\alpha}\right)\normF{\mG^{t}  - \mV^{t}}^2 + \frac{6\gamma^2 C + 18L^2\eta^2}{\alpha}\normF{\mV^t-\bar\vv^t\vone^\top}^2 \\
    & \quad + \frac{18\gamma^2 C L^2}{\alpha}\normF{\mX^t-\mH^t}^2 + \frac{18\gamma^2 C L^2}{\alpha}\normF{\mX^t-\bar\vx^t\vone^\top}^2 + \frac{18L^2\eta^2 n}{\alpha}\norm{\bar\vv^t}^2 + \frac{12 n \sigma^2}{b\alpha}.
\end{align*}
Plugging in the definitions of $\Omega_i^t$, we obtain \eqref{eq:omega-2}.

 \paragraph{Bounding $\Omega_3^t$ in \eqref{eq:omega-3}}
To bound the consensus error $\normF{\mX^{t+1}-\bar\vx^{t+1}\vone^\top}^2$, by the  update rule of \beer (cf. Line \ref{line:update-x}), we have
\begin{align}
 &   \normF{\mX^{t+1}-\bar\vx^{t+1}\vone^\top}^2 \nonumber \\
 & = \normF{\mX^t + \gamma \mH^t (\mW - \mI)- \eta \mV^t - \bar\vx^t \vone^\top + \eta\bar\vv^t \vone^\top}^2 \nonumber \\
    &\overset{\mathrm{(i)}}{=} \normF{ \mX^t \widetilde\mW - \bar\vx^t \vone^\top + \gamma (\mH^t - \mX^t) (\mW - \mI)- \eta \mV^t + \eta\bar\vv^t \vone^\top}^2 \nonumber \\
    & \overset{\mathrm{(ii)}}{\le} (1+\beta)(1-\gamma \rho)\normF{\mX^t - \bar\vx^t \vone^\top}^2  + \left(1+\frac{1}{\beta}\right)\left(2\gamma^2\normF{(\mH^t - \mX^t) (\mW - \mI)}^2 + 2\eta^2\normF{\mV^t-\bar\vv^t\vone^\top}^2\right) \nonumber \\
    & \overset{\mathrm{(iii)}}{\le}  \left(1-\frac{\gamma \rho}{2}\right)\normF{\mX^t - \bar\vx^t \vone^\top}^2  + \left(1+\frac{2}{\gamma\rho}\right)\left(2\gamma^2\normF{(\mH^t - \mX^t) (\mW - \mI)}^2 + 2\eta^2\normF{\mV^t-\bar\vv^t\vone^\top}^2\right) \nonumber \\
    & \overset{\mathrm{(iv)}}{\le}  \left(1-\frac{\gamma \rho}{2}\right)\normF{\mX^t - \bar\vx^t \vone^\top}^2  + \left(1+\frac{2}{\gamma\rho}\right)\left(2\gamma^2C\normF{\mH^t - \mX^t}^2 + 2\eta^2\normF{\mV^t-\bar\vv^t\vone^\top}^2\right) \nonumber \\
    & \le \left(1-\frac{\gamma \rho}{2}\right)\normF{\mX^t - \bar\vx^t \vone^\top}^2 + \frac{6\gamma C}{\rho}\normF{\mH^t - \mX^t}^2 + \frac{6\eta^2}{\gamma\rho}\normF{\mV^t-\bar\vv^t\vone^\top}^2, \nonumber
\end{align}
where (i) follows from the definition $\widetilde\mW = \mI + \gamma(\mW - \mI)$, (ii) follows from applying Young's inequality twice and Lemma~\ref{lem:prop_mixing}, i.e.
$$\normF{ \mX^t \widetilde\mW - \bar\vx^t \vone^\top } \leq (1-\gamma \rho)\normF{\mX^t - \bar\vx^t \vone^\top}^2 ,$$
(iii) follows by choosing $\beta = \gamma\rho/2$, and (iv) uses the definition of $C$ (cf.~\eqref{eq:C}). Plugging in the definitions of $\Omega_i^t$, we obtain \eqref{eq:omega-3}.

\paragraph{Bounding $\Omega_4^t$ in \eqref{eq:omega-4}}
First, note that
\begin{align*}
    \normF{\mV^{t+1} - \bar\vv^{t+1}\vone^\top}^2 & = \normF{\mV^{t+1} - \bar\vv^{t}\vone^\top + \bar\vv^{t}\vone^\top - \bar\vv^{t+1}\vone^\top}^2 \\
    & = \normF{\mV^{t+1} - \bar\vv^{t}\vone^\top}^2 - n\norm{\bar\vv^{t+1} - \bar\vv^{t}}^2 \\
    & \le \normF{\mV^{t+1} - \bar\vv^{t}\vone^\top}^2.
\end{align*}
Thus by the update rule of \beer (cf. Line \ref{line:update-v}), we have
\begin{align*}
    &~~~~ \normF{\mV^{t+1} - \bar\vv^{t+1}\vone^\top}^2\\
    & \le \normF{\mV^{t+1} - \bar\vv^{t}\vone^\top}^2 \\
    & = \normF{\mV^{t} + \gamma \mG^{t+1}(\mW - \mI) + \tilde\nabla_b F(\mX^{t+1}) - \tilde\nabla_b F(\mX^t) - \bar\vv^{t}\vone^\top}^2 \\
    & = \normF{( \mV^{t} \widetilde\mW - \bar\vv^{t}\vone^\top) + \gamma (\mG^{t+1} - \mV^{t})(\mW - \mI) + (\tilde\nabla_b F(\mX^{t+1}) - \tilde\nabla_b F(\mX^t))}^2 \\
    &  \overset{\mathrm{(i)}}{\le}  \left(1-\frac{\gamma\rho}{2}\right)\normF{\mV^{t} - \bar\vv^{t}\vone^\top}^2 + \left(1+\frac{2}{\gamma\rho}\right)\left(2\gamma^2 C\normF{\mG^t-\mV^t}^2 + 2L^2\normF{\mX^{t+1} - \mX^t}^2 + \frac{4n\sigma^2}{b}\right) \\
    & \overset{\mathrm{(ii)}}{\le}  \left(1-\frac{\gamma\rho}{2}\right)\normF{\mV^{t} - \bar\vv^{t}\vone^\top}^2 + \frac{6\gamma C}{\rho}\normF{\mG^t-\mV^t}^2 + \frac{6 L^2}{\gamma\rho}\normF{\mX^{t+1} - \mX^t}^2 + \frac{12 n \sigma^2}{b\gamma\rho}\\
    & \le \left(1-\frac{\gamma\rho}{2} + \frac{18 L^2\eta^2}{\gamma\rho}\right)\normF{\mV^{t} - \bar\vv^{t}\vone^\top}^2 + \frac{6\gamma C}{\rho}\normF{\mG^t-\mV^t}^2 \\
    & \quad + \frac{18\gamma C L^2}{\rho}\normF{\mX^t-\mH^t}^2 + \frac{18\gamma C L^2}{\rho}\normF{\mX^t-\bar\vx^t\vone^\top}^2 + \frac{18n\eta^2 L^2}{\gamma\rho}\norm{\bar\vv^t}^2 + \frac{12 n \sigma^2}{b\gamma\rho},
\end{align*}
where (i) and (ii) are obtained similarly as the derivation of \eqref{eq:omega-3}, and the last line follows from \eqref{eq:diff_x}.
Thus, we can get \eqref{eq:omega-4} by plugging in the definitions of $\Omega_i^t$ and conclude the proof.
\end{proof}

\subsection{Proof of Theorem \ref{thm:nonconvex-full} and \ref{thm:nonconvex}}
\label{sec:proof-ncv}
Note that Theorem \ref{thm:nonconvex} is a strict generalization of Theorem \ref{thm:nonconvex-full}, and thus we will directly prove Theorem \ref{thm:nonconvex}.
This proof makes use of Lemma \ref{lem:gt} and Lemma~\ref{lem:omega-1}, by constructing some proper Lyapunov function and demonstrate its descending property using a linear system argument, which is also used in, e.g., \citet{li2021destress,liao2021compressed}.

\paragraph{Step 1: establishing a ``descent'' property of the function value}
First, we have the following inequality captures the ``descent'' of the function value.
\begin{align}
    f(\bar\vx^{t+1}) & \overset{\mathrm{(i)}}{ \le} f(\bar\vx^{t}) - \eta\dotp{\bar\vv^t}{\nabla f(\bar\vx^{t})} + \frac{\eta^2 L}{2}\norm{\bar\vv^{t}}^2 \nonumber \\
    & = f(\bar\vx^{t}) - \frac{\eta}{2}\norm{\nabla f(\bar\vx^{t})}^2 - \frac{\eta}{2}\norm{\bar\vv^{t}}^2 + \frac{\eta}{2}\norm{\nabla f(\bar\vx^{t}) - \bar\vv^t}^2 + \frac{\eta^2 L}{2}\norm{\bar\vv^{t}}^2 \nonumber \\
    & = f(\bar\vx^{t}) - \frac{\eta}{2}\norm{\nabla f(\bar\vx^{t})}^2 + \frac{\eta}{2}\norm{\nabla f(\bar\vx^{t}) - \bar\vv^t}^2 - \left(\frac{\eta}{2} - \frac{\eta^2 L}{2}\right)\norm{\bar\vv^{t}}^2 \nonumber \\
    & \overset{\mathrm{(ii)}}{ \le} f(\bar\vx^{t}) - \frac{\eta}{2}\norm{\nabla f(\bar\vx^{t})}^2 + \frac{\eta}{2n^2}\norm{\nabla F(\bar\vx^t)\vone - \tilde\nabla_b F(\mX^t)\vone}^2 - \left(\frac{\eta}{2} - \frac{\eta^2 L}{2}\right)\norm{\bar\vv^{t}}^2 \nonumber  \\
    & = f(\bar\vx^{t}) - \frac{\eta}{2}\norm{\nabla f(\bar\vx^{t})}^2 + \frac{\eta}{2n^2}\norm{\nabla F(\bar\vx^t)\vone - \nabla F(\mX^t)\vone}^2 - \left(\frac{\eta}{2} - \frac{\eta^2 L}{2}\right)\norm{\bar\vv^{t}}^2
    \nonumber \\
    &~~~+ \frac{\eta}{2n^2}\norm{\nabla F(\mX^t)\vone - \tilde\nabla_b F(\mX^t)\vone}^2
    + \frac{\eta}{n^2} \left\langle \nabla F(\mX^t)\vone - \tilde\nabla_b F(\mX^t)\vone, \nabla F(\bar\vx^t)\vone - \nabla F(\mX^t)\vone \right\rangle , \notag
\end{align}
where (i) comes from the $L$-smooth assumption (Assumption \ref{ass:smooth}), (ii) comes from Lemma \ref{lem:gt}.
Take expectation on both sides, and using the bounded variance assumption (Assumption \ref{ass:bounded-variance}) and independence of stochastic samples, we get
\begin{align}
    \E f(\bar\vx^{t+1})
    & \le \E f(\bar\vx^{t}) - \frac{\eta}{2} \E \norm{\nabla f(\bar\vx^{t})}^2 + \frac{\eta}{2n^2} \E \norm{\nabla F(\bar\vx^t)\vone - \nabla F(\mX^t)\vone}^2 - \left(\frac{\eta}{2} - \frac{\eta^2 L}{2}\right) \E \norm{\bar\vv^{t}}^2 + \frac{\eta \sigma^2}{2bn} \nonumber \\
    &  \overset{\mathrm{(i)}}{ \le} \E f(\bar\vx^{t}) - \frac{\eta}{2}\E \norm{\nabla f(\bar\vx^{t})}^2 + \frac{\eta }{2n}\E \normF{\nabla F(\mX^t) - \nabla F(\bar\vx^t)}^2 - \left(\frac{\eta}{2} - \frac{\eta^2 L}{2}\right)\E \norm{\bar\vv^{t}}^2 + \frac{\eta \sigma^2}{2bn}\nonumber \\
    &  \overset{\mathrm{(ii)}}{ \le} \E f(\bar\vx^{t}) - \frac{\eta}{2}\E \norm{\nabla f(\bar\vx^{t})}^2 + \frac{\eta  L^2}{2n} \E \normF{\mX^t - \bar\vx^t\vone^\top}^2 - \left(\frac{\eta}{2} - \frac{\eta^2 L}{2}\right)\E \norm{\bar\vv^{t}}^2 + \frac{\eta \sigma^2}{2bn}, \nonumber
\end{align}
where (i) comes from Young's inequality, and (ii) comes from the $L$-smooth assumption (Assumption \ref{ass:smooth}) again. 
Finally, by substituting definitions of $\Omega_3^t$ and $\Omega_5^t$,
we reach
\begin{equation}\label{eq:descent}
\E f(\bar\vx^{t+1}) \le \E f(\bar\vx^t) - \frac{\eta}{2}\E\norm{\nabla f(\bar\vx^t)}^2 + \frac{\eta L^2}{2n}\Omega_3^t - \left(\frac{\eta}{2} - \frac{\eta^2 L}{2}\right)\Omega_5^t + \frac{\eta \sigma^2}{2bn}.
\end{equation}

\paragraph{Step 2: constructing the Lyapunov function}  By representing
\begin{align}
\bm{\Omega}^t = [   \Omega_1^{t} \quad    \Omega_2^{t} \quad   \Omega_3^{t} \quad  \Omega_4^{t} ]^{\top} ,
\end{align}
Lemma~\ref{lem:omega-1} can be written more compactly as
\begin{align}\label{eq:ls}
\bm{\Omega}^{t+1}& \leq \underbrace{ \begin{bmatrix}
1-\frac{\alpha}{2} + \frac{6\gamma^2C}{\alpha}  & 0 &  \frac{6\gamma^2 C}{\alpha} &   \frac{6\eta^2}{\alpha} \\
\frac{18\gamma^2 C L^2}{\alpha} &  1-\frac{\alpha}{2} + \frac{6\gamma^2 C}{\alpha} & \frac{18\gamma^2 C L^2}{\alpha} & \frac{6\gamma^2 C+18L^2\eta^2}{\alpha}  \\
 \frac{6\gamma C}{\rho} & 0 & 1-\frac{\gamma\rho}{2} &  \frac{6\eta^2}{\gamma \rho} \\
  \frac{18\gamma C L^2}{\rho} &  \frac{6\gamma C}{\rho} & \frac{18\gamma CL^2}{\rho} &1-\frac{\gamma \rho}{2} + \frac{18L^2\eta^2}{\gamma\rho}
\end{bmatrix} }_{=: \bm{A}}  \bm{\Omega}^{t}  +  \underbrace{ \begin{bmatrix}
\frac{6n\eta^2}{\alpha}  \\ 
 \frac{18L^2\eta^2 n}{\alpha} \\
0 \\
\frac{18n\eta^2 L^2}{\gamma\rho} 
\end{bmatrix} }_{=: \bm{b}_1}   \Omega_5^t +  \underbrace{ \begin{bmatrix}
0 \\ 
 \frac{12 n }{ \alpha}\\
0 \\
 \frac{12 n  }{ \gamma\rho}
\end{bmatrix} }_{=: \bm{b}_2}  \frac{\sigma^2}{b} .
\end{align}

Define the Lyapunov function
\begin{align}
\Phi_t  & = \E f(\bar\vx^t) - f^* + \frac{c_1L }{n}\cdot\Omega_1^t + \frac{c_2 \rho^2}{nL}\cdot\Omega_2^t+\frac{c_3 L }{n}\cdot\Omega_3^t+\frac{c_4 \rho^2}{nL}\Omega_4^t \nonumber \\
& = \E f(\bar\vx^t) - f^*  + \vs^{\top}\bm{\Omega}^t, \label{eq:lynapunov_def}
\end{align}
where 
$$\vs =\left[ \frac{c_1L }{n} \quad  \frac{c_2 \rho^2}{nL} \quad \frac{c_3 L }{n} \quad  \frac{c_4 \rho^2}{nL} \right] $$ 
for some constants $c_1,c_2,c_3,c_4$ that will be specified later.

By \eqref{eq:ls} from Lemma~\ref{lem:omega-1} and the descent property \eqref{eq:descent}, we have
\begin{align}
\Phi_{t+1}  &= \E f(\bar\vx^{t+1}) - f^* +  \vs^{\top}\bm{\Omega}^{t+1} \nonumber \\
& \leq  \E f(\bar\vx^t) - f^* - \frac{\eta}{2}\E \norm{\nabla f(\bar\vx^t)}^2 + \frac{\eta L^2}{2n}\Omega_3^t - \left(\frac{\eta}{2} - \frac{\eta^2 L}{2}\right)\Omega_5^t + \frac{\eta \sigma^2}{2bn}   +  \vs^{\top}\left(\bm{A} \bm{\Omega}^{t} + \Omega_5^t\bm{b}_1+\frac{\sigma^2}{b} \bm{b}_2 \right) \label{eq:descent_lyapunov} \\
& \leq \Phi_t - \frac{\eta}{2}\E \norm{\nabla f(\bar\vx^t)}^2 - \left(\frac{\eta}{2} - \frac{\eta^2 L}{2}\right)\Omega_5^t  + \frac{\eta \sigma^2}{2bn}  +( \vs^{\top} \bm{A}  -\vs^{\top} +\vq^{\top} )\bm{\Omega}^{t} +  \vs^{\top}(\Omega_5^t\bm{b}_1+\bm{b}_2\frac{\sigma^2}{b})\nonumber \\
& = \Phi_t - \frac{\eta}{2}\E \norm{\nabla f(\bar\vx^t)}^2 +( \vs^{\top} \bm{A}  -\vs^{\top} +\vq^{\top} )\bm{\Omega}^{t}  - \left(\frac{\eta}{2} - \frac{\eta^2 L}{2} - \bm{s}^{\top}\bm{b}_1 \right)\Omega_5^t  + \left(  \frac{\eta }{2n} + \vs^{\top} \bm{b}_2  \right) \frac{  \sigma^2}{b} ,\nonumber
\end{align}
 where $\vq = [0 \quad 0 \quad \frac{\eta L^2}{2n}\quad 0]^{\top}$. For a moment we assume that there exist some constants $c_1,c_2,c_3,c_4 > 0$ such that
 \begin{subequations}\label{eq:tbp}
 \begin{align}
 \vs^{\top} ( \bm{A} -\bm{I})  + \vq^{\top} & \leq \bm{0}, \\
 \frac{\eta}{2} - \frac{\eta^2 L}{2} - \bm{s}^{\top}\bm{b}_1 & \geq 0, 
 \end{align}
 \end{subequations}
 leading to
 \begin{align*} 
 \Phi_{t+1}  & \leq \Phi_t - \frac{\eta}{2}\E \norm{\nabla f(\bar\vx^t)}^2 +\left(  \frac{\eta }{2n} + \vs^{\top} \bm{b}_2  \right) \frac{  \sigma^2}{b} \leq  \Phi_t - \frac{\eta}{2}\E \norm{\nabla f(\bar\vx^t)}^2 + \frac{36c_4  \sigma^2}{c_{\gamma} L b\alpha}. 
 \end{align*} 
  The proof is thus completed by recursing the above relation over $t=0,\ldots, T-1$.

\paragraph{Step 3: verifying \eqref{eq:tbp}} It boils down to verify  \eqref{eq:tbp} is feasible, and it is equivalent to verify there exist parameters $c_1,c_2,c_3,c_4,\gamma,\eta > 0$ satisfying the following matrix inequality:
\begin{align*}
    \begin{bmatrix}
    \mI - \mA^\top \\
    -\vb_1
    \end{bmatrix}
    \text{diag}\left[\frac{L}{n}, \frac{\rho^2}{nL}, \frac{L}{n}, \frac{\rho^2}{nL}\right]
    \begin{bmatrix}
    c_1 \\
    c_2 \\
    c_3 \\
    c_4 
    \end{bmatrix}
    \ge
    \begin{bmatrix}
    \vq \\
    \frac{\eta^2L}{2} - \frac{\eta}{2}
    \end{bmatrix}.
\end{align*}
Note that by choosing $\gamma = c_{\gamma}\rho\alpha, \eta = c_\eta\gamma\rho^2/L$, and
setting $c_{\gamma} \le \frac{1}{6\sqrt{C}}$ and $c_{\eta}\le \frac{1}{9}$, we get 
\begin{equation}\label{eq:par_bound}
1-\frac{\alpha}{2} + \frac{6\gamma^2C}{\alpha} \le 1-\frac{\alpha}{4},\quad 1-\frac{\gamma \rho}{2} + \frac{18L^2\eta^2}{\gamma\rho} \le 1-\frac{\gamma\rho}{4}.
\end{equation}
Now, it suffices to show that there exist $c_1 ,c_2 ,c_3,c_4,c_{\gamma},c_{\eta} > 0$ such that the following inequalities are satisfied:
\begin{align*}
    \begin{bmatrix}
    \frac{\alpha L}{4n} & -\frac{18 c_{\gamma}^2 \alpha \rho^4 L}{n} & -\frac{6 c_{\gamma} \alpha L}{n} & -\frac{18 C c_{\gamma} \alpha \rho^2 L}{n} \\
    0 & \frac{\alpha \rho^2}{4 n L} & 0 & -\frac{6 C c_{\gamma} \alpha \rho^2}{n L} \\
    -\frac{6 c_{\gamma} \rho \gamma C L}{n} & -\frac{18 c_{\gamma} \rho \gamma L}{n} & \frac{\gamma \rho L}{2 n} & -\frac{18 C \gamma \rho L}{n} \\
    -\frac{6 c_{\eta}^2 c_{\gamma} \gamma \rho^3 }{n L} & -\frac{6 C \gamma \rho^3 (1+3c_{\eta}^2\rho^4)}{n L} & -\frac{6 c_{\eta}^2 \gamma \rho^3}{n L} & \frac{\gamma\rho^3}{4 n L} \\
    -12 c_{\eta} c_{\gamma} \rho^3 \frac{\eta}{2} & -36 c_{\eta} c_{\gamma} \rho^5 \frac{\eta}{2} & 0 & -36 c_{\eta} \rho \frac{\eta}{2}
    \end{bmatrix} \begin{bmatrix}
    c_1 \\
    c_2 \\
    c_3 \\
    c_4 
    \end{bmatrix} \geq \begin{bmatrix}
    0 \\
    0 \\
    \frac{c_{\eta} \gamma \rho L}{2n} \\
    0 \\
    \left(c_{\eta} c_{\gamma} \alpha \rho^3 - 1\right)\frac{\eta}{2}
    \end{bmatrix}.
\end{align*}
Given $\alpha \le 1, \rho \le 1$, this can be further reduced to show the existence of $c_1 ,c_2 ,c_3,c_4,c_{\gamma},c_{\eta} > 0$ such that
\begin{align*}
\begin{bmatrix}
1 & -72Cc_\gamma^2 & -24Cc_\gamma & -72Cc_\gamma \\
0 & 1 & 0  &-24Cc_\gamma \\
-12Cc_\gamma & -35Cc_\gamma & 1 & -36C \\
-24c_\eta^2 c_\gamma & -24c_\gamma(1+3c_\eta^2) & -24c_\eta^2 & 1 \\
- 12c_\eta c_\gamma  & -36c_\eta c_\gamma & 0 & -36 c_\eta
\end{bmatrix} \begin{bmatrix}
c_1 \\
c_2 \\
c_3 \\
c_4 
\end{bmatrix} \geq \begin{bmatrix}
0 \\
0 \\
c_\eta \\
0 \\
-1+c_\eta c_\gamma
\end{bmatrix} .
\end{align*}
This can be easily verified by noting that as long as $c_\eta$ and $c_\gamma$ are set sufficiently small, it is straightforward to find feasible $c_1,c_2,c_3,c_4$.

\subsection{Proof of Theorem \ref{thm:pl-full} and \ref{thm:pl}}
\label{sec:proof-pl}
Since Theorem \ref{thm:pl} is a generalization of Theorem \ref{thm:pl-full}, it suffices to directly prove Theorem \ref{thm:pl}. The proof strategy of Theorem \ref{thm:pl} is similar to that of Theorem \ref{thm:nonconvex}. However, in order to achieve the advertised linear convergence rate under the PL condition (Assumption \ref{ass:pl}), we need to use a slightly different linear system.
 
Denote $\kappa: = L/\mu$. Taking the same Lyapunov function $\Phi_t$ in \eqref{eq:lynapunov_def}, by the same argument of Section~\ref{sec:proof-ncv} up to \eqref{eq:descent_lyapunov}, we have
 \begin{align*} 
\Phi_{t+1}  & \leq  \E f(\bar\vx^t) - f^* - \frac{\eta}{2}\E \norm{\nabla f(\bar\vx^t)}^2 + \frac{\eta L^2}{2n}\Omega_3^t - \left(\frac{\eta}{2} - \frac{\eta^2 L}{2}\right)\Omega_5^t + \frac{\eta \sigma^2}{2bn}   +  \vs^{\top}\left(\bm{A} \bm{\Omega}^{t} + \Omega_5^t\bm{b}_1+\frac{\sigma^2}{b} \bm{b}_2 \right) \\
& \leq    (1-\eta\mu)( \E f(\bar\vx^t) - f^* ) +( \vs^{\top} \bm{A}   +\vq^{\top} )\bm{\Omega}^{t}  - \left(\frac{\eta}{2} - \frac{\eta^2 L}{2} - \bm{s}^{\top}\bm{b}_1 \right)\Omega_5^t  + \left(  \frac{\eta }{2n} + \vs^{\top} \bm{b}_2  \right) \frac{  \sigma^2}{b}  \\
& = (1-\eta\mu)\Phi_{t} +\left( \vs^{\top} \bm{A} -(1-\eta\mu) \vs^{\top}  +\vq^{\top} \right)\bm{\Omega}^{t}  - \left(\frac{\eta}{2} - \frac{\eta^2 L}{2} - \bm{s}^{\top}\bm{b}_1 \right)\Omega_5^t  + \left(  \frac{\eta }{2n} + \vs^{\top} \bm{b}_2  \right) \frac{  \sigma^2}{b},
 \end{align*} 
where $\vq = [0 \quad 0 \quad \frac{\eta L^2}{2n}\quad 0]^{\top}$, and the second inequality follows from the PL condition (Assumption \ref{ass:pl}). If we can establish that there exist there exist some constants $c_1,c_2,c_3,c_4$ such that 
 \begin{subequations}\label{eq:tbp-pl}
 \begin{align}
 \vs^{\top} ( \bm{A} -(1-\eta\mu) \bm{I})  + \vq^{\top} & \leq \bm{0}, \\
 \frac{\eta}{2} - \frac{\eta^2 L}{2} - \bm{s}^{\top}\bm{b}_1 & \geq 0, 
 \end{align}
 \end{subequations}
we arrive at
 \begin{align*}  
 \Phi_{t+1}  & \leq ( 1- \eta\mu) \Phi_t   +\left(  \frac{\eta }{2n} + \vs^{\top} \bm{b}_2  \right) \frac{  \sigma^2}{b}  \leq ( 1- \eta\mu) \Phi_t  + \frac{36c_4  \sigma^2}{c_{\gamma} L b\alpha}.
 \end{align*} 
 Recursing the above relation then complete the proof.

It then boils down to establish \eqref{eq:tbp-pl}. By similar arguments as Section~\ref{sec:proof-ncv}, in view of \eqref{eq:par_bound} and $\alpha \le 1, \rho \le 1$, it is sufficient to show there exist constants $c_1 ,c_2 ,c_3,c_4,c_{\gamma},c_{\eta} > 0$ such that
 \begin{align*}
\begin{bmatrix}
1 - \frac{4c_\eta c_\gamma}{\kappa} & -72Cc_\gamma^2 & -24Cc_\gamma & -72Cc_\gamma \\
0 & 1- \frac{4c_\eta c_\gamma}{\kappa} & 0  &-24Cc_\gamma \\
-12Cc_\gamma & -35Cc_\gamma & 1 - \frac{2c_\eta }{\kappa} & -36C \\
-24c_\eta^2 c_\gamma & -24c_\gamma(1+3c_\eta^2) & -24c_\eta^2 & 1- \frac{4c_\eta }{\kappa} \\
- 12c_\eta c_\gamma  & -36c_\eta c_\gamma & 0 & -36 c_\eta
\end{bmatrix} \begin{bmatrix}
c_1 \\
c_2 \\
c_3 \\
c_4 
\end{bmatrix} \geq \begin{bmatrix}
0 \\
0 \\
c_\eta \\
0 \\
-1+c_\eta c_\gamma
\end{bmatrix} .
\end{align*}
This can be easily verified by noting that as long as $c_\eta$ and $c_\gamma$ are set sufficiently small, it is straightforward to find feasible $c_1,c_2,c_3,c_4$.

\end{document}